%% file: main.tex
\documentclass[11pt]{article}
\pdfoutput=1

\usepackage{ifthen}
\newboolean{usenatbib}
\setboolean{usenatbib}{true}

\ifthenelse{\boolean{usenatbib}}{
  \usepackage{natbib}
}{
  \usepackage[style=alphabetic,maxalphanames=12,maxcitenames=2,maxbibnames=99,giveninits=false,doi=false,url=false,natbib,backend=biber]{biblatex}
  \addbibresource{refs.bib}

  \DeclareDelimFormat[bib]{nametitledelim}{\addspace}
}

\usepackage{bm}
\usepackage{fullpage}
\usepackage[english]{babel}
\usepackage[utf8]{inputenc}
\usepackage[T1]{fontenc}
\usepackage{ae} \usepackage{aecompl}
\usepackage{enumitem}
\setlist{nosep}
\usepackage{microtype}
\usepackage{booktabs}
\usepackage{hhline}
\usepackage{makecell}
\usepackage{etoolbox}
\apptocmd{\sloppy}{\hbadness 10000\relax}{}{}

\usepackage{amsmath,amsthm,amssymb,mathtools,thmtools}
\usepackage{graphicx}
\usepackage[textsize=tiny]{todonotes}
\usepackage[colorlinks=true, allcolors=blue]{hyperref}
\usepackage{algorithm,algorithmicx,algpseudocode}
\usepackage{xfrac}
\usepackage[normalem]{ulem}

\hypersetup{colorlinks,
            breaklinks=true,
            linkcolor=blue,
            citecolor=blue,
            urlcolor=magenta,
            linktocpage,
            plainpages=false,
            bookmarks=false}

\usepackage{notations}
\usepackage{comment}
\usepackage{tikz,pgfplots,pgfplotstable}
\usepgfplotslibrary{units}
\usepackage{tabularx} 
\usepackage{subfig}
\usepackage{array}
\usepgfplotslibrary{groupplots}

\title{Agnostic Private Density Estimation for GMMs \\ via List Global Stability}

\author{}

\numberwithin{equation}{section}

\begin{document}

\author{
    Mohammad Afzali\thanks{McMaster University, \texttt{afzalikm@mcmaster.ca}}
    \and 
    Hassan Ashtiani\thanks{McMaster University, \texttt{zokaeiam@mcmaster.ca}. Hassan Ashtiani is also a faculty affiliate at the Vector Institute and supported by an NSERC Discovery Grant.}
     \and 
    Christopher Liaw\thanks{Google, \texttt{cvliaw@google.com}.}
}

\maketitle

\begin{abstract}%

  We consider the problem of private density estimation for mixtures of unrestricted high dimensional Gaussians in the agnostic setting. We prove the first upper bound on the sample complexity of this problem. Previously, private learnability of high dimensional GMMs was only known in the realizable setting \citep{afzali2024mixtures}.

To prove our result, we exploit the notion of \textit{list global stability} \citep{ghazi2021user,ghazi2021sample} that was originally introduced in the context of private supervised learning. We define an agnostic variant of this definition, showing that its existence is sufficient for agnostic private density estimation. We then construct an agnostic list globally stable learner for GMMs. 
\end{abstract}

\allowdisplaybreaks
\input{intro}
\input{prelim}
\input{reduction_proof}

\input{mixtures}
\input{GMM}
\input{related}

\ifthenelse{\boolean{usenatbib}}{
  \bibliographystyle{plainnat}
  \bibliography{refs.bib}
}{
  \printbibliography
}

\input{appendix}

\end{document}

%% file: intro.tex
\begin{section}{Introduction}

     Density estimation is a fundamental problem that has been studied by statisticians for decades. In this problem, we have access to i.i.d. samples from an unknown distribution $f$ that belongs to a known class of distributions $\cF$. The goal is to find a distribution $\hat{f} \in \cF$ that is close to $f$ with respect to the total variation distance (i.e., $L_1$ distance). This setting, often referred to as the \emph{realizable} setting, is unrealistic in most cases, since the true distribution might not be exactly a member of the class $\cF$\footnote{Unless $\cF$ is chosen to be a very rich class of distributions, which would make the estimation problem harder.}. This might happen due to model misspecification or adversarial corruptions.   
     For example, we expect a density estimation method for Gaussians to perform well even if $f$ is only approximately a Gaussian.     
 As opposed to the realizable setting, in the \emph{agnostic} setting, we do not assume that $f$ belongs to $\cF$; instead, the goal is to find a distribution $\hat{f} \in \cF$ that is as close to $f$ as possible compared to the ``best'' distribution in $\cF$.

\newpage
   \begin{definition}[Agnostic Density Estimation]\label{def:agnostic}
        Let $\cF$ be a class of distributions, $C\geq 1$, and $\alpha,\beta \in (0,1)$. An algorithm $\cA$ is said to be a $C$-agnostic $(m,\alpha,\beta)$-learner for $\cF$ if:
        \begin{quote}
            For every distribution $g$, after receiving an i.i.d.~sample set $S$ of size at least $m$ from $g$, the algorithm outputs a distribution $\hat{f}=\cA(S)$, such that $\dtv(\hat{f},g)\leq C \cdot OPT+\alpha$ with probability at least $1-\beta$, where $\dtv$ is the total variation distance (see Section~\ref{sec:prelim}), and $OPT=\inf_{f\in \cF}\dtv(g,f)$ measures how far $g$ is from the class $\cF$.
        \end{quote}
    \end{definition}

    Designing agnostic (and more generally \emph{robust}) density estimators has been the subject of extensive studies in the literature, and several useful tools have been developed for it, such as the Minimum Distance Estimator \citep{yatracos1985rates, devroye2001combinatorial} and robust compression schemes \citep{ashtiani2020near}.

    We study the problem of agnostic density estimation under the constraint of differential privacy \citep{dwork2006calibrating,dwork2006our}, which is the gold standard for protecting individuals' privacy in a dataset. At a high level, differential privacy requires the algorithm's outputs on every two neighbouring datasets to be statistically indistinguishable from each other (see Definition~\ref{def:privacy}).

    Given the importance of robustness and privacy, one may ask whether it is possible to design estimators that are both robust and private. In the context of supervised learning, it is known that agnostic supervised learning can often be reduced to learning in the realizable setting~\citep{hopkins2022realizable}.
    In fact, \citet{alon2020closure} have shown that one can always turn a private classifier that works in the realizable setting into a private classifier that works in the agnostic setting. Therefore, designing sample-efficient agnostic classifiers does not seem to be a particularly challenging task. This picture, however, is completely different in the density estimation setting. For example, it has been shown that there are classes of distributions that are privately learnable in the realizable setting but not in the agnostic setting~\citep{ben2024distribution}. Therefore, there is no general recipe to convert non-robust density estimators to robust ones---either in the private or the non-private setting. 
    
    The above observation raises the question of whether there is a framework for designing \emph{private and agnostic} density estimators for the commonly used classes of distributions, such as Gaussians and their mixtures. 

    For the case of high dimensional Gaussians, private learnability is well understood \citep{karwa2018finite, kamath2019privately,bun2019private,biswas2020coinpress,aden2021sample, hopkins2022efficient,kamath2022private,ashtiani2022private, kothari2022private, alabi2023privately, hopkins2023robustness}. In fact, some of these results establish the private learnability of high dimensional Gaussians with unrestricted parameters in the \emph{agnostic setting} \citep{aden2021sample, ashtiani2022private, kothari2022private, alabi2023privately, hopkins2023robustness}. See Section~\ref{section:related} for a more thorough discussion of these results and other related work. For the case of GMMs, however, the problem is much more challenging.      

    \subsection{Private density estimation for GMMs}

    Consider the class of mixtures of $k$ Gaussians with unrestricted parameters (i.e., means, covariances, and mixing weights) in $d$ dimensions.
    One of the factors that makes private learning of GMMs challenging is the ``identifiability'' issue: unlike Gaussians, two GMMs with very close densities can have very different parameters. As a result, even non-private \emph{parameter estimation} for GMMs requires an exponential number of samples in terms of the number of components \citep{moitra2010settling}. In contrast, non-private \emph{density estimation} for GMMs can be done with a polynomial number of samples in terms of $k$ and $d$ \citep{devroye2001combinatorial,ashtiani2018sample,ashtiani2018nearly,ashtiani2020near}. These sample-efficient density estimators are therefore inevitably inaccurate (and practically unstable) in terms of the \emph{parameters} they recover. As a result, some of the standard approaches that are used for private learning of Gaussians do not extend to GMMs. For example, one cannot directly resort to robust-to-private reductions \citep{hopkins2023robustness,asi2023robustness} since they only work in a finite-dimensional (parameter) space. Similarly, the private to non-private reduction for GMM parameter estimation of \citet{arbas2023polynomial} cannot be directly applied.

    Another avenue for design of private density estimators for GMMs is the framework of private hypothesis selection \citep{bun2019private}. In fact, if the Gaussians have bounded parameters and bounded condition numbers then one can build a ``finite cover'' for GMMs and apply this framework. There is, however, a significant obstacle in extending this approach to GMMs with unrestricted parameters. Namely, one would need a ``locally small'' cover for GMMs, which is hard if not impossible to construct\footnote{Proposition B.6 in \cite{aden2021privately} demonstrates that GMMs do not admit a locally small cover (the way that a locally small cover is defined in \citet{bun2019private}).}.

    The private learnability of GMMs in the univariate (and axis-aligned) setting was established by \cite{aden2021privately} in the realizable setting. The use of stability-based histograms \citep{bun2019simultaneous} for detecting heavy hitters makes this approach infeasible for handling unrestricted high dimensional GMMs.

    Recently, the private learnability of high dimensional GMMs has been established in the \emph{realizable setting}~\citep{afzali2024mixtures}. 
    One of the main ideas that they exploit is that although parameter estimation for GMMs requires an exponential number of samples, ``list decoding'' parameters can be done with a polynomial number of samples\footnote{Ignoring the components with negligible mixing weights.}. They then run these list decoders on multiple sub-samples and privately aggregate the results. The aggregation is done via an advanced form of heavy hitter selection in the space of parameters.  
    However, it is not easy to extend their approach to the agnostic setting. In particular, when the samples are heavily corrupted, the lists outputted by list decoders may not share an (approximately) common member. 
    This raises the following question.

\begin{quote}
    Can mixtures of high dimensional Gaussians with unrestricted parameters be privately learned in the agnostic setting using a polynomial number of samples?
\end{quote}

        In this paper, we resolve the above question by leveraging a form of stability in the design of private algorithms.
        There are several notions of stability that are related to differential privacy including \emph{global stability} \citep{thakurta2013differentially,bun2020equivalence}, \emph{(list) replicability} \citep{impagliazzo2022reproducibility, chase2023stability}, and \emph{list global stability} \citep{ghazi2021sample, ghazi2021user}. In particular, we find the notion of {list global stability} suitable for our application\footnote{ See Appendix~\ref{sec::stability} for a discussion of these notions, and why other notions are not desirable in our setting.}.
    
\begin{definition}[List Global Stability~\citep{ghazi2021sample, ghazi2021user}]\label{def:stable-list-decoding-simple}
    For $m, L\in \bN$, a list decoding algorithm $\cA$ receives a sample $S$ of size $m$ (from an input domain) and outputs a list $H_S$ of size $L$. 
    We say a $\cA$ is $(m,\rho, L)$-list-globally-stable if for every distribution $\cD$ over input, there exists a hypothesis $h_{\cD}$ such that
    $\probs{S\sim \cD^m}{h_{\cD} \in H_S}\geq \rho$.
\end{definition}

In the context of classification, \citet{ghazi2021sample} describe how to convert a list globally stable learner into a private learner with a number of samples that is logarithmic in $L$. Although this reduction is stated for the realizable setting, it automatically implies an agnostic learner (recall that for private classification, agnostic learnability can be reduced its realizable counterpart \citep{alon2020closure}). However, as discussed before, such a general reduction is impossible for density estimation. Can we still use list globally stability for agnostic and private learning of GMMs?

\subsection{Our contributions}
Our contributions are twofold: \textit{(i)} we show a reduction from agnostic private density estimation to list globally stable learning, and \textit{(ii)} we design a list globally stable learner for the class of GMMs, establishing their agnostic private learnability. First, let us define list globally stable learning in the context of agnostic density estimation.

\begin{definition}[List Global Stability for Agnostic Density Estimation]\label{def:stable-list-decoding}
     Let $m, L\in \bN$, $\alpha \in (0,1)$, $C>1$, and $\cF$ be class of distributions. We say $\cA$ is a $(C, \alpha)$-accurate $(m, \rho, L)$-list-globally-stable learner for $\cF$ if for every distribution $g$ (not necessarily in $\cF$) there exists distribution $\tilde{g}$ such that 
    \begin{quote}
           \begin{tabular}{p{8cm} p{6cm}}
    \textit{(1)} $\cA$ is a list globally stable algorithm; &  $\probs{S\sim g^m}{\tilde{g} \in \cA(S)}\geq \rho$ \\
     \textit{(2)} $\cA$ satisfies agnostic utility guarantee; & $\dtv(\tilde{g}, g) < \alpha + C.\inf_{f\in \cF}\dtv(g,f)$
    \end{tabular}
    \end{quote}
\end{definition}

    Our first contribution is to show that given a list globally stable learner for a class of distributions, one can privately learn that class of distributions in the \emph{agnostic} setting (as defined in Def.~\ref{def:agnostic}).

\begin{restatable}[Private agnostic learning via list global stability]{theorem}{mainprivaterobust}\label{thm:main-private-agnostic}
      Let $\cF$ be a class of distributions. For any $m,L\in \bN$, $\alpha,\beta \in (0,1), C > 1$, if $\cF$ is $(C, \frac{\alpha}{3+4C})$-accurate $(m,0.91, L)$-list-globally-stable learnable,
     then $\cF$ is $(\eps,\delta)$-privately $7C$-agnostic $(n,\alpha,\beta)$-learnable with the following number of samples:
        \[
    n=\tilde{O}\left(\frac{\log(L/\delta\beta)}{\eps} \cdot (m+\frac{\log(L/\beta)}{\alpha^2})\right)
        \]
\end{restatable}

    Our reduction is, to some extent, similar to that of \citet{ghazi2021sample} for private classification, but with key differences. At a high level, they run the list globally stable algorithm on sub-samples, identify and remove bad candidates (i.e., those with non-zero empirical error) from each list, and detect repeated candidates in the lists using the \emph{sparse selection} technique \citep{ghazi2020differentially}. In the agnostic distribution learning setting, however, we cannot filter bad candidates right away. The reason is that whether a candidate is considered ``bad'' depends on the level of corruption, and the level of corruption is not known and is hard to estimate\footnote{For agnostic classification, the error of ERM is a good proxy for the corruption level; but for agnostic density estimation, the corruption level cannot be estimated even without privacy. In fact, given samples from an unknown distribution over real line, it is hard to estimate the total variation of the underlying distribution from the standard normal distribution.}. Therefore, we instead use a recursive procedure to filter out bad distributions while ensuring (1) privacy, (2) the quality of the remaining candidates, and (3) the existence of at least one repeated good candidate among the remaining candidates. We then use the private selection method of \citet{beimel2013private,bun2015differentially} to select a good candidate. An overview of our technique, along with the formal proof and algorithms, is provided in Section~\ref{section:main-proof}.

    With the reduction of Theorem~\ref{thm:main-private-agnostic} at hand, the remaining critical question is: can we design an effective list globally stable learner for GMMs? The first observation is that the elements outputted by the list globally stable learner have to be somehow ``discretized'' otherwise the same element may not repeat exactly (as required by Def.~\ref{def:stable-list-decoding-simple} and Def.~\ref{def:stable-list-decoding}).


    Designing a list globally stable learner with good utility is challenging even in the realizable setting. However, a substantially more difficult task is addressing the agnostic setting. Here, the list globally stable algorithm has to be  \textit{(i)} ``stable'' and \textit{(ii)} have good utility even when some samples are corrupted (i.e., when the underlying distribution is not a GMM). 
     With these corruptions, however, the list may not contain any candidate that is ``super close'' to the underlying GMM. Still, the list should contain (with high probability) a specific ``stable'' candidate (depending only on the underlying GMM $g$) that is sufficiently close to $g$ (the distance depends on the amount of corruption). 

    We circumvent the aforementioned challenges and demonstrate that the class of GMMs indeed admits a list globally stable learner (in the agnostic setting).  In particular, we prove that \emph{(i)} Gaussians admit a list globally stable learner (see Lemma~\ref{lemma:stable-gd}), and \emph{(ii)} the class of mixtures of any list globally stable learnable class, is also list globally stable learnable (see Theorem~\ref{thm:stable-list-mix}). Therefore, we conclude that GMMs are list globally stable learnable and can be integrated into our agnostic private learning framework.
    
\begin{theorem}[Private agnostic learning GMMs, informal version of \ref{thm:main-gmm-agnostic}]
    \label{inf:thm:main-gmm-agnostic}
     Let $\alpha,\delta \in (0,1)$ and $\eps \geq 0$. The class of mixtures of $k$ unrestricted $d$ dimensional Gaussians is $(\eps,\delta)$-privately $21$-agnostic $(n,\alpha,0.99)$-learnable with 
       \[  n=\tilde{O}\left(\frac{k^2d^4\log(1/\delta)}{\alpha^2\eps}\right)  .
       \]
\end{theorem}


It is worth mentioning that our agnostic result also improves the realizable result of \citet{afzali2024mixtures} in terms of sample complexity by a factor of $1/\alpha^2$.  

\subsection{Paper organization}
We define some notations in Section~\ref{sec:prelim} before stating our techniques and proofs. Section~\ref{section:main-proof} presents the high-level proof idea, the formal proof, and the pseudo-code for our general reduction. In Section~\ref{section:mixtures}, we develop a useful tool for list globally stable learning mixture distributions. In Section~\ref{section:GMMS}, we prove the list global stability of Gaussians and their mixtures, and conclude with the agnostic private learnability of GMMs. Finally, we review some related work in Section~\ref{section:related}.

\end{section}

%% file: prelim.tex
\begin{section}{Preliminaries}\label{sec:prelim}
For a set $\cF$, define $\cF^k= \cF  \times \dots \times \cF$ ($k$ times), and $\cF^*=\bigcup_{k=1}^{\infty} \cF^k$. For two absolutely continuous densities $f_1(x),f_2(x)$ on $\bR^d$, the total variation (TV) distance is defined as $\dtv(f_1,f_2)=\frac{1}{2}\int_{\bR^d}|f_1(x)-f_2(x)|\, \dd x$. In this paper, if $\dist:\cF\times \cF \rightarrow \bR_{\geq 0}$ is a metric, $f \in \cF$, and $\cF' \subseteq \cF$ then we define $\dist(f, \cF') = \inf_{f' \in \cF'} \dist(f, f')$.

\begin{definition} [$\alpha$-cover]
    A set $C_{\alpha} \subseteq \cF$ is said to be an $\alpha$-cover for a metric space $(\cF,\dist)$, if for every $f \in \cF$, we have $\dist(f,C_{\alpha})\leq \alpha$.
\end{definition}

\begin{definition}[$k$-mixtures] \label{def:k-mixtures}
    Let $\cF$ be an arbitrary class of distributions. We denote the class of $k$-mixtures of $\cF$ by $\kmix(\cF)= \Delta_k \times \cF^k$ where $\Delta_k = \{ w \in \bR^k \,:\, w_i \geq 0, \sum_{i=1}^k w_i = 1\}$ is the $(k-1)$-dimensional probability simplex.
\end{definition}
In this paper, we simply write $f \in \kmix(\cF)$ to denote a $k$-mixture from the class $\cF$ where the representation of the weights is implicit.
Further, if $f \in \kmix[k_1](\cF)$ and $g \in \kmix[k_2](\cF)$ then we write $\dtv(f, g)$ to denote the TV distance from the underlying distribution.

\begin{definition} [Unbounded Gaussians] \label{def:general_gaussian}
    Let $\cG_d=\{\cN(\mu, \Sigma): \mu \in \mathbb{R}^d, \Sigma \in S^d\}$ be the class of d-dimensional Gaussians, where $S^d$ is the set of positive-definite cone in $\bR^{d\times d}$.
\end{definition}

The following result on learning a finite class of distributions is based on the Minimum Distance Estimator~\citep{yatracos1985rates}; see the excellent book by~\citet{devroye2001combinatorial} for details.
\begin{theorem} [Learning finite classes, Theorem 6.3 of \citet{devroye2001combinatorial}] \label{thm:mde}
        Let $\cF$ be a finite class of distributions, and $\alpha,\beta\in (0,1)$. Then, $\cF$ is $3$-agnostic $(n,\alpha,\beta)$-learnable with  $n = O(\frac{\log|\cF|+\log(1/\beta)}{\alpha^2})$.  
\end{theorem}

\begin{subsection}{Differential privacy}
    Two datasets $D,D'\in \cX^n$ are called neighbouring datasets if they differ by one element. We use the notation $D \sim D'$ to denote that they are neighboring datasets. Informally, a differentially private algorithm is required to have close output distributions on neighbouring datasets.

    \begin{definition}[$(\eps,\delta)$-Indistinguishable]
        Two distribution $f_1,f_2$ with support $\cX$ are said to be $(\eps,\delta)$-indistinguishable if for all measurable subsets $E \in \cX$, $\probs{X\sim f_1}{X\in E} \leq e^\eps \probs{X\sim f_2}{X\in E} +\delta$ and $\probs{X\sim f_2}{X\in E} \leq e^\eps \probs{X\sim f_1}{X\in E} +\delta$.
    \end{definition}
    \begin{definition}[$(\eps,\delta)$-Differential Privacy \citep{dwork2006calibrating,dwork2006our}] \label{def:privacy}
        A randomized algorithm $\cM:\cX^n \to \cY$ is said to be $(\eps,\delta)$-differentially private if for every two neighbouring datasets $D,D'\in \cX^n$, the output distributions $\cM(D),\cM(D')$ are $(\eps,\delta)$-indistinguishable.
    \end{definition}

We utilize the property of differentially private algorithms that guarantees the preservation of differential privacy when these algorithms are composed adaptively. By adaptive composition, we refer to executing a series of algorithms $\cM_1(D),\cM_2(D),\cdots,\cM_k(D)$, where the selection of the algorithm $\cM_i$ can depend on the outputs of the preceding algorithms $\cM_1(D),\cM_2(D),\cdots,\cM_{i-1}(D)$.

\begin{lemma}[Composition of DP \citep{dwork2006calibrating,dwork2006our}]\label{lemma:composition}
    If algorithms $\cM_1,\cM_2,\cdots,\cM_k$ are $(\eps_1,\delta_1),(\eps_2,\delta_2),\cdots,(\eps_k,\delta_k)$-differentialy private, and $\cM$ is an adaptive composition of $\cM_1,\cdots,\cM_k$, then $\cM$ is $(\sum_{i\in[k]}\eps_i,\sum_{i\in[k]}\delta_i)$-DP.
\end{lemma}

Consider the problem of private selection, where we are given a set of candidates and a score function to measure how ``good'' each candidate is with respect to a given data set.
A well known method for privately selecting a ``good'' candidate is the Exponential Mechanism \citep{mcsherry2007mechanism}.

Nonetheless, the Exponential Mechanism may not have a good utility when the number of candidates is very large or infinite. Another useful tool for this task, is therefore the Choosing Mechanism of \citet{beimel2013private,bun2015differentially}, which is compatible with infinite candidate sets. The Choosing Mechanism ensures returning a ``good'' candidate, as long as the score function has a bounded growth.

\begin{algorithm}
\caption{Choosing Mechanism}\label{alg:choosing-mech}
\begin{algorithmic}[1]
\Require $D\in \cX^T$, candidate set $\cF$, quality function $\score\colon \cF \times \cX^T \to \bR_{\geq 0}$, parameters $\beta,\eps,\delta,k$.
\State Set $\text{MAX}=\max_{f\in \cF} \score(f,D)$
\State Set $\widetilde{\text{MAX}}=\text{MAX} + \text{Lap}(\frac{4}{\eps})$.
\State If $\widetilde{\text{MAX}} \leq \frac{8}{\eps}\log(\frac{4k}{\beta\eps\delta})$, reject and return $\bot$.
\State Set $G=\{f \in \cF: \score(f,D)\geq 1\}$.
\State \Return $\hat{f}\in G$ with probability $\propto \exp(\eps \cdot \score(\hat{f},D)/4)$.
\end{algorithmic}
\end{algorithm}

\begin{theorem}[Choosing Mechanism, Lemma 3.8 of \citet{bun2015differentially}] \label{thm:choosing}
    Let $(\cF,\kappa)$ be a metric space and $\cX$ be an arbitrary set. Let $\score\colon \cF \times \cX^T \to \bR_{\geq 0}$ be a function such that:
    \begin{itemize}
        \item $\score(f,\emptyset)=0$ for all $f\in \cF$.
        \item If $D'=D\cup \{x\}$, then $\score(f,D)+1\geq \score(f,D') \geq \score(f,D)$ for all $f\in \cF$.
        \item There are at most $k$ values of $f\in \cF$ such that $\score(f,D)+1= \score(f,D')$.
    \end{itemize}

    Then algorithm \ref{alg:choosing-mech} 
     is $(\eps,\delta)$-DP with the following property. For every $D \in \cX^T$, with probability at least $1-\beta$, it outputs $\hat{f}\in \cF$ satisfying:
\begin{equation*}
    \score(\hat{f},D)\geq \max_{f\in \cF} \score(f,D)-\frac{16}{\eps} \log(\frac{4kT}{\beta \eps \delta}).
\end{equation*}
\end{theorem}

We will also use the Truncated Laplace distribution, which will be useful for privately answering threshold queries in our algorithm.
\begin{definition}[Truncated Laplace distribution]\label{def:tlaplace} Let $\delta\in (0,1)$, and $\eps,\Delta > 0$. Truncated Laplace distribution is denoted by $\text{TLap}(\Delta,\eps,\delta)$ with the following density function:
    \[
    f_{\text{TLap}(\Delta,\eps,\delta)} (x)=\begin{cases}
    \frac{\eps}{2\Delta(1-e^{-\eps R/\Delta})} e^{-\eps |x|/\Delta} &  x\in [-R,R]\\
    0 & x\not\in [-R,R].
  \end{cases}
    \]
    where $R=\frac{\Delta}{\eps}\log(1+\frac{e^\eps-1}{2\delta})$.
\end{definition}

\begin{lemma}[Theorem 1 of \citet{geng2018truncated}] \label{lemma:tlaplace}
    Let $\delta \in (0,1), \eps>0$, and $q:\cX^n \rightarrow \bR^d$ be a function, define the sensitivity of $q$ to be $\Delta:= \max_{D\sim D'\in \cX^n} ||q(D)-q(D')||_1$. Then $q(x)+Y$ is $(\eps,\delta)$-DP, where $Y\sim \text{TLap}(\Delta,\eps,\delta)$.
\end{lemma}
    
\end{subsection}
\end{section}

%% file: reduction_proof.tex
\begin{section}{From list global stability to agnostic private density estimation} \label{section:main-proof}
In this section we prove Theorem~\ref{thm:main-private-agnostic}. We show that if a class of distributions is list globally stable learnable, then it is privately learnable in the agnostic setting. Let us restate the theorem.

\mainprivaterobust*

First, we give a high-level overview of Algorithm~\ref{alg:private-agnostic-learning}. The procedure consists of three steps:

\textbf{Step I.} We split the dataset (drawn i.i.d.~from an unknown distribution $g$) into $T$ disjoint subsets. We run the list globally stable learner for $\cF$ on each subset and get $T$ lists of distributions. Given the definition of the list globally stable learner (see Definition~\ref{def:stable-list-decoding}), we can conclude that there exists a distribution $\tilde{g}$ which is in a large fraction of lists with high probability. In other words, $\tilde{g}$ is a ``stable'' output.

\textbf{Step II.} Next, we filter the lists and ensure that all distributions are close to the true distribution $g$ w.r.t.~$\dtv$. This can be done using the MDE algorithm (see Theorem~\ref{thm:mde}). Within each list, we find a distribution $\hat{g}$ that is close to the true distribution $g$ w.r.t.~$\dtv$, and remove any other distribution that is far from $\hat{g}$. The filtering should be done in a way that does not remove all ``stable'' distributions from the lists (one of which is $\tilde{g}$). Thus, the filtering radius needs to be chosen carefully since we are not aware of the distance $\dtv(g,\cF)$. We will describe how to privately choose a suitable filtering radius using binary search combined with the Propose-Test-Release framework (see Algorithm~\ref{alg:find-filter}). 

\textbf{Step III.} Finally, we consider the candidate set $\cM$ to be the union of all filtered lists (which is a subset of $\cF$) and assign a score to each candidate to measure how ``stable'' that candidate is w.r.t.~the lists. We know there exists at least one stable candidate which is not filtered from the lists since $\cA$ is an algorithm that ``preserves agnostic utility guarantee'' (see Definition~\ref{def:stable-list-decoding}). Also, we know that the size of each list is bounded, which allows us to conclude there are not too many ``stable'' candidates. In other words, we design the score function to take advantage of this fact and have a bounded growth (in the sense of Theorem~\ref{thm:choosing}). This allows us to privately select a ``stable'' candidate using the Choosing Mechanism (see Theorem~\ref{thm:choosing}). The second (filtering) step ensures that the selected candidate is close to $g$ w.r.t.~$\dtv$.

    \begin{algorithm}
\caption{Private Agnostic Density Estimation}\label{alg:private-agnostic-learning}
\begin{algorithmic}[1]
\Require Parameters $\alpha,\beta,\delta\in(0,1)$, $\eps>0$, $C \geq 1$, $m,L\in \bN$, a $(C,\frac{\alpha}{3+4C})$-accurate $(m,0.91,L)$-list-globally-stable learner $\text{SLD}$ for $\cF$.
\Ensure A distribution $f\in \cF$ (or $\bot$).

\State Let $\alpha'=\frac{\alpha}{3+4C}, \eps' = \frac{\eps}{1+ \log(1/\alpha')},\delta' = \frac{\delta}{1+ \log(1/\alpha')}, \beta'= \frac{\beta \eps'}{7680\log(\frac{9830400L}{\eps'^3\beta\delta'})\log(1/\alpha')}$, $m_1= m+\frac{\log(L/\beta')}{\alpha'^2}$, and $T= \frac{640}{\eps'} \log(\frac{1280~L}{\beta'\delta'\eps'^2})$.

\State Draw a sample set $D$ of size $T\cdot m_1$, and randomly partition it to $T$ disjoint data sets $D_1,D_2,\cdots,D_T$.

\For{$i \in [T]$} 
\State Randomly partition $D_i$ into two disjoint subsets $D_i^1$ of size $m$, and $D_i^2$ of size $\frac{\log(L/\beta')}{\alpha'^2}$.
\State {$\cL_i = \text{SLD}(D_i^1)$} \Comment{Run the list globally stable learner on $D_i^1$.}
\State {$\hat{f}_i = \text{MDE}(\cL_i,D_i^2,\alpha',\beta')$} \Comment{Find a ``good'' distribution in $\cL_i$ using MDE algorithm on $D_i^2$.}
\EndFor
\State
 \LeftComment{Filter $\cL_i$'s to get ``good'' candidates along with their ``stability'' as their scores.}
\State $(\cM,\score)=\text{Get-Candidates-And-Scores}(\alpha',\delta',\eps',C,T,D,\{\hat{f}_1, \cdots, \hat{f}_T\}, \{\cL_1,\cdots,\cL_T\} )$

\State
\LeftComment{Run Choosing Mechanism to select a ``good'' and ``stable'' candidate.}
\State \Return Choosing-Mechanism$\left( \cM,D,\score,\beta',\eps',\delta',L\right)$
\end{algorithmic}
\end{algorithm}

\begin{algorithm}
\caption{Get Candidates And Scores}\label{alg:find-filter}
\begin{algorithmic}[1]
\Require Parameters $\alpha',\delta'\in(0,1)$, $\eps'>0$, $C \geq 1$, $T\in \bN$, a data set $D \in \cX^*$, a set of distributions $\{\hat{f}_1, \hat{f}_2, \cdots, \hat{f}_T\} \in \cF^T$, and a set of lists $\{\tilde{\cL}_1,\tilde{\cL}_2,\cdots,\tilde{\cL}_T\} \in (\cF^*)^T$.
\Ensure A set of filtered distributions $\cM \subseteq \cF$, and a score function $\score:\cM\rightarrow \bR$.
\State Let $LP= 0, UP= 1$. 
\State Let $\text{state}=\emptyset, \score = \emptyset, \cM = \emptyset$.
\While {$UP-LP>\alpha'$}
\State Let $\widetilde{\text{OPT}} = \frac{LP+UP}{2}$
\State Let $(\text{state}, \cM, \score) =\text{Filter-And-Test}(\alpha',\delta',\widetilde{\text{OPT}},\eps',C,T,D,\{\hat{f}_1, \cdots, \hat{f}_T\}, \{\cL_1,\cdots,\cL_T\} )$
\If {$\text{state}=\text{``reject''}$}
\State Let $LP = \widetilde{\text{OPT}}$
\Else 
\State Let $UP = \widetilde{\text{OPT}}$
\EndIf

\EndWhile

\State \Return $(\cM, \score)$
\end{algorithmic}
\end{algorithm}

    \begin{algorithm}
\caption{Filter and Test}\label{alg:filter-test}
\begin{algorithmic}[1]
\Require Parameters $\alpha',\delta',\widetilde{\text{OPT}}\in(0,1)$, $\eps'>0$, $C \geq 1$, $T\in \bN$, a data set $D\in \cX^*$, a set of distributions $\{\hat{f}_1, \hat{f}_2, \cdots, \hat{f}_T\} \in \cF^T$, and a set of lists $\{\cL_1,\cL_2,\cdots,\cL_T\}\in (\cF^*)^T$.

\Ensure $\text{State}\in\{\text{``reject'', ``accept''}\}$, a set of filtered distributions $\cM \subseteq \cF$, and a score function $\score:\cM\rightarrow \bR$.
\For{$i \in [T]$} 
\State {$\tilde{\cL}_i = \{f\in\cL_i: \dtv(f,\hat{f}_i)\leq 4C\cdot \widetilde{\text{OPT}}+2\alpha'\}$} \Comment{Filter out ``bad'' distributions in $\cL_i$}
\EndFor

\State Let $\cM = \bigcup_{i \in [T]}\tilde{\cL}_i$
\For{$f \in\cM$} \LeftComment{The score of each element $f \in \cM$ is defined w.r.t. the data set $D$ which is implicit in the construction of $\tilde{\cL}_i$'s}
\State $\score(f,D) = |\{i \in [T]\,:\, f\in \tilde{\cL}_i \}|$. 
\EndFor

\State Let $\text{MAX} = \max_{f \in \cM} \score(f,D)$.
\State Let $\widetilde{\text{MAX}}= \text{MAX} + \text{TLap}(1,\eps',\delta')$.

\If{$\widetilde{\text{MAX}} < 0.8T + \frac{1}{\eps'}\log(1+\frac{e^{\eps'}}{2\delta'})$}
\State \Return (``reject'', $\cM$, $\score: \cM \rightarrow \bN$)
\Else 
\State \Return (``accept'', $\cM$, $\score: \cM \rightarrow \bN$)
\EndIf
\end{algorithmic}
\end{algorithm}

\textbf{Proof of Theorem~\ref{thm:main-private-agnostic}:}
\begin{proof}
We will first prove the utility and then the privacy.

\textbf{Utility analysis.}
        Let $g$ be the true distribution, $\rho=0.91$, $\alpha'=\frac{\alpha}{3+4C}, \eps' = \frac{\eps}{1+ \log(1/\alpha')},\delta' = \frac{\delta}{1+ \log(1/\alpha')}$, and $\beta' \in (0,1)$ be a parameter to be set later. We want to $(\eps,\delta)$-privately output a distribution $\hat{f}\in \cF$, such that with probability at least 
$1-\beta$, $\dtv(\hat{f},g)\leq 7C\cdot\dtv(g,\cF)+\alpha$. We do this in three steps:

\textbf{Step I.} Non-private candidate generation:
  
    Let $m_1\geq m+\frac{\log(L/\beta')}{\alpha'^2}$, and $T\geq \max\{\frac{640}{\eps'} \log(\frac{1280~L}{\beta'\delta'\eps'^2}),\frac{20}{\eps'}\log(1+\frac{e^{\eps'}}{2\delta'})\}$. Consider $T$ disjoint data sets $D_1,D_2,\cdots, D_T$ each of size $m_1$ drawn i.i.d.~from $g$.
    For each data set $D_i$, use $m$ samples to run the $(C,\alpha')$-accurate $(m,\rho,L)$-list-globally-stable learner for $\cF$. Let $\cL_i$ be the outputted list. From Definition~\ref{def:stable-list-decoding}, we know that there exists a distribution $\tilde{g}$ satisfying $\probs{S\sim g^m}{\tilde{g} \in\cL_i}\geq \rho$ and $\dtv(\tilde{g},g)\leq C\cdot\dtv(g,\cF)+\alpha'$. Let $X_i=\mathbf{1}\{\tilde{g}\in \cL_i\}$ be a binary random variable. Then using Hoeffding's inequality, we have
    \[
        \prob{\sum_{i\in [T]} X_i \leq \mathbb{E}\left[\sum_{i\in [T]}X_i\right] - t}\leq \exp(-\frac{2t^2}{T}).
    \]
    Substituting $t=0.01T\rho$ results in
    $\prob{\sum_{i\in [T]} X_i \leq 0.99T\rho} \leq \exp(-2.10^{-4}T\rho^2)$. Since $T \geq \frac{\log(1/\beta')}{\rho^2}$, we get that with probability at least $1-\beta'$, $\tilde{g}$ is in at least $0.99\rho$ fraction of lists. Let the indices of such lists be the set $\cI\subseteq [T]$. 

    \textbf{Step II.} Non-private candidate filtering:
    
    Recall that, the size of each data set $D_i$ is at least $m_1\geq m+\frac{\log(L/\beta')}{\alpha'^2}$. Now, we use the remaining $\frac{\log(L/\beta')}{\alpha'^2}$ samples from $D_i$ to filter out ``bad'' distributions in $\cL_i$ w.r.t.~$\dtv$. The filtering process works as follows; Theorem~\ref{thm:mde} (MDE) implies that for every $i\in \cI$, using $\frac{\log(L/\beta')}{\alpha'^2}$ samples, we are able to find a distribution $\hat{f}_i \in \cL_i$ such that w.p. at least $1-\beta'$,
    \begin{align}
    \dtv(\hat{f}_i,g)&\leq 3\dtv(g,\cL_i)+\alpha' \\
    &\leq 3(\dtv(g,\tilde{g})+\dtv(\tilde{g},\cL_i))+\alpha' \\
    &\leq 3(C\cdot\dtv(g,\cF)+0)+\alpha'     \\
    &\leq 3C\cdot\dtv(g,\cF)+\alpha'
    \end{align}
     Now we remove any distribution $f$ from $\cL_i$ that is far from $\hat{f}_i$. The filtering radius depends on a value which we call $\widetilde{\text{OPT}}$. In Algorithm~\ref{alg:find-filter}, we show how to choose a suitable filtering radius. We will later describe this process. For now, let the filtered list to be:
     \begin{align}
         \tilde{\cL}_i = \{f\in\cL_i: \dtv(f,\hat{f}_i)\leq 4C\cdot \widetilde{\text{OPT}}+2\alpha'\}  \label{proof-main-filtering-step}
     \end{align}
 
    As a result, we get $T$ lists $\tilde{\cL}_1,\cdots,\tilde{\cL}_T$ each of size at most $L$. Let $\text{OPT}=\dtv(g,\cF)$. We are not aware of this value since $g$ is unknown.
     We claim that, if $\text{OPT}\leq \widetilde{\text{OPT}}$, then for every $i \in \cI$,  $\tilde{\cL_i}$ still contains $\tilde{g}$. This is because:
     \begin{align}
         \dtv(\tilde{g},\hat{f_i})&\leq \dtv(\tilde{g},g)+\dtv(g,\hat{f_i})\\
         &\leq (C\cdot\dtv(g,\cF)+\alpha')+(3C\cdot\dtv(g,\cF)+\alpha') \\
         &\leq 4C\cdot\dtv(g,\cF)+2\alpha'\\
         &\leq 4C\cdot\widetilde{\text{OPT}}+2\alpha' \label{eq:notfiltered}
     \end{align}
     hence $\tilde{g}$ is not filtered from $\cL_i$, and we have $\tilde{g}\in\tilde{\cL}_i$. 
     
     Furthermore, if $\widetilde{\text{OPT}}\leq \text{OPT}+\alpha'$, then for every $i \in \cI$, $\tilde{\cL}_i$'s members are ``good'' w.r.t. $\dtv$, since for any $f\in \tilde{\cL}_i$ it holds that:
    \begin{align}
        \dtv(f,g)&\leq \dtv(f,\hat{f}_i)+\dtv(\hat{f}_i,g)\\
        &\leq (4C\cdot\widetilde{\text{OPT}}+2\alpha')+(3C\cdot\dtv(g,\cF)+\alpha')\\
        &\leq 7C\cdot\text{OPT}+(3+4C)\alpha'. \label{line:good-members}
    \end{align}

    Let the ``score'' function be defined as $\score(f,D) = |\{i \in [T]\,:\, f\in \tilde{\cL}_i \}|$ (see Algorithm~\ref{alg:filter-test}). We have: 
\begin{itemize}
     \item $\score(f,\emptyset)=0$ for all $f \in \cF$.
        \item If $D'=D\cup \{\cL\}$, then $\score(f,D)+1\geq \score(f,D') \geq \score(f,D)$ for all $f\in \cF$. Since each list $\cL_i\in D$ contributes to any $f$'s score by at most 1.
        \item There are at most $k=L$ many distributions $f \in \cF$ such that $\score(f,D)+1= \score(f,D')$.
\end{itemize}
At a high level, the score of each distribution (candidate) represents its ``stability'' with respect to the filtered lists. We will use this score function to privately select a distribution (candidate) with high score using the Choosing Mechanism~\ref{alg:choosing-mech}.

From Eq.~\ref{eq:notfiltered}, we can say that if $\text{OPT} \leq \widetilde{\text{OPT}}$, then $\tilde{g}$ is still in a large fraction of the lists and receives a large score.
At this point, we privately check whether the maximum score is large enough; this will be used to determine if we have chosen a suitable filtering radius $\widetilde{\text{OPT}}$. If the filtering radius is small, then the maximum score would be low (and we cannot claim the output is indeed a "good" distribution. Recall that for $i \in [T] \setminus \cI$, members of $\tilde{\cL}_i$ may be "bad" and have low scores.). On the other hand, if we choose the filtering radius to be large, the maximum score would be large, but at the same time, it will affect the utility of the algorithm since we are not completely filtering "bad" distributions from the $\cL_i$'s where $i \in \cI$ (which potentially can have large scores). In Algorithm~\ref{alg:find-filter}, we show how to adaptively choose a suitable $\widetilde{\text{OPT}}$ using the Filter-And-Test subroutine~\ref{alg:filter-test}.

From Eq.~\ref{eq:notfiltered}, recall that, if $\text{OPT}\leq \widetilde{\text{OPT}}$, then for every 
$i \in \cI$, $\tilde{\cL}_i$ contains $\tilde{g}$, and we have $\text{MAX}\geq \score(\tilde{g},D) \geq 0.99\rho T>0.9T$ (with probability at least $1-2T\beta'$ over the correctness of the list globally stable learner and MDE). Using the fact that $T \geq \frac{20}{\eps'}\log(1+\frac{e^{\eps'}}{2\delta'})$, we get that with probability 1, $\widetilde{\text{MAX}} \geq \text{MAX} - |\text{TLap}(1,\eps',\delta')| > 0.8 T + |\text{TLap}(1,\eps',\delta')| = 0.8T +\frac{1}{\eps'}\log(1+\frac{e^{\eps'}}{2\delta'})$ (see Definition~\ref{def:tlaplace}). This means that for all values of $\widetilde{\text{OPT}}$ that satisfy $\text{OPT}\leq \widetilde{\text{OPT}}$, the Filter-And-Test subroutine~\ref{alg:filter-test} does not output ``reject'' (with probability at least $1-2T\beta'$).

Now using a binary search technique in Algorithm~\ref{alg:find-filter}, we are able to find a filtering value $\widetilde{\text{OPT}}$, such that the Filter-And-Test subroutine~\ref{alg:filter-test} does not output ``reject''.
Putting together with a union bound over $\log(1/\alpha')$ iterations of the binary search algorithm, we can say with probability at least $1-2T \log(1/\alpha') \beta'$ we have $LP\leq \text{OPT}$, and $\widetilde{\text{OPT}}\leq \alpha'+\text{OPT}$. 

Note that, the above argument does not necessarily imply that $\text{OPT}\leq \widetilde{\text{OPT}}$. However, in the Filter-And-Test subroutine~\ref{alg:filter-test}, if $\text{MAX}< 0.8T$, then with probability 1, we have $\widetilde{\text{MAX}} < 0.8T +\frac{1}{\eps'}\log(1+\frac{e^{\eps'}}{2\delta'})$ and the subroutine outputs ``reject''.
Thus, when the Filter-And-Test subroutine~\ref{alg:filter-test} does not output ``reject'', we have $\text{MAX}\geq 0.8T$ which is enough for us to provide utility guarantee. In other words, we still have a reasonably ``stable'' candidate in the filtered lists (which might be different from $\tilde{g}$).

\textbf{Step III.} Private selection:

Now, let $\hat{f}= \text{ChoosingMechanism}\left( \cF,D,\score,\beta',\eps',\delta',k\right)$\footnote{In the last line of the Algorithm~\ref{alg:private-agnostic-learning}, we used $\cM \subseteq \cF$ as the set of candidates, which is algorithmically more efficient. The reason is that the Choosing Mechanism only considers candidates with non-zero scores, which are the members of $\cM$. However, in the analysis, for compatibility with Theorem~\ref{thm:choosing}, we consider the set of all candidates $\cF$.}. Using  Theorem~\ref{thm:choosing} together with a union bound, implies with probability at least $1-(2T\log(1/\alpha')+1)\beta'$:
\begin{align*}
    \score(\hat{f},D)\geq \max_{f \in \cF} \score(f,D)-\frac{16}{\eps'} \log(\frac{4kT}{\beta' \eps' \delta'}) \geq 0.8 T -\frac{16}{\eps'} \log(\frac{4kT}{\beta' \eps' \delta'}) \geq 0.7 T.
\end{align*}
In the last inequality we used Claim~\ref{claim:choose_inequality} with $ c_1=\frac{160}{\eps'}, c_2=\frac{160}{\eps'}\log(\frac{4L}{\eps'\beta' \delta'})$ to get $\frac{16}{\eps'} \log(\frac{4kT}{\beta' \eps' \delta'})\leq 0.1T$.\\
Thus, with probability at least $1-(2T\log(1/\alpha')+1)\beta'\geq 1- 3T\log(1/\alpha')\beta'$, we have
\begin{align*}
    \score(\hat{f},D)=|\{i \in [T]\,:\, \hat{f}\in \tilde{\cL}_i \}| \geq 0.7 T
\end{align*}    

    Now that we chose a suitable filtering radius satisfying $\widetilde{\text{OPT}}\leq \alpha+\text{OPT}$, using Eq.~\ref{line:good-members}, we know that for $i \in \cI$, $\tilde{\cL}_i$'s members are ``good'' w.r.t~$\dtv$. This means that, at most $\frac{|[T]\setminus \cI|}{T}\leq 0.01\rho \leq 0.001$ fraction of lists may contain ``bad'' distributions w.r.t.~$\dtv$. This implies that any ``bad'' distribution could have the score of at most $0.001T$. Therefore,  
     $\hat{f}$ whose score is at least $0.7T$, is indeed a ``good'' distribution and for some $i \in \cI$, $\hat{f}$ belongs to $ \tilde{\cL}_i$:
     \begin{align} \label{proof-main-final-dist}
         \dtv(\hat{f},g)&\leq \max_{f\in \tilde{\cL}_i}\dtv(f,g) 
         \leq 7C\cdot\text{OPT}+(3+4C)\alpha'.
     \end{align}

\textbf{Privacy analysis.} Note that:
\begin{itemize}
    \item All iterations of the binary search in Algorithm~\ref{alg:find-filter} are $(\eps',\delta')$-DP because of the privacy guarantee of Truncated Laplace Mechanism (see Theorem~\ref{lemma:tlaplace}), and the fact that the sensitivity of the $\score$ function is 1.
    \item The $\score$ function satisfies all properties in Theorem~\ref{thm:choosing}, which implies that the Choosing Mechanism used in Algorithm~\ref{alg:private-agnostic-learning} is also $(\eps',\delta')$-DP.
\end{itemize}
      Putting together, the Algorithm~\ref{alg:private-agnostic-learning} is $\left((1+\log(1/\alpha))\eps',(1+\log(1/\alpha))\delta'\right)$-private due to the composition property of differential privacy (see Lemma~\ref{lemma:composition}).

    Now, we substitute $\alpha'=\frac{\alpha}{3+4C}, \beta'= \frac{\beta \eps'}{7680\log(\frac{9830400L}{\eps'^3\beta\delta'})\log(1/\alpha')}, \eps' = \frac{\eps}{1+\log(1/\alpha')}$, and $\delta' = \frac{\delta}{1+\log(1/\alpha')}$. Final calculations (see Claim~\ref{claim:choose_inequality}\footnote{With $c_1=\frac{1920}{\eps'\beta},c_2=~\frac{c_1}{\beta}\log(\frac{1280L}{\eps'^2\delta'})$.}), implies that with probability at least $1-3T\log(1/\alpha)\beta'\geq 1-\beta$, it holds that $\dtv(\hat{f},g)\leq 7C\cdot\dtv(g,\cF)+\alpha$. The total sample complexity is:
    \[
        T\cdot m_1 = \frac{640}{\eps'} \log(\frac{1280~L}{\beta'\delta'\eps'^2})\cdot (m+\frac{\log(L/\beta')}{\alpha'^2})= \tilde{O}\left(\frac{\log(L/\delta\beta)}{\eps} \cdot (m+\frac{\log(L/\beta)}{\alpha^2})\right). \qedhere
    \]
\end{proof}

\end{section}

%% file: mixtures.tex
\section{List globally stable learning of mixture distributions} \label{section:mixtures}
In this section, we develop a tool for list globally stable learning mixture distributions. At a high level, we show that if a class of distributions is list globally stable learnable, then the class of its mixtures is also list globally stable learnable. This task is challenging since some components of the mixture might be heavily corrupted, while others may have negligible weights and be difficult to recover.

\begin{theorem}\label{thm:stable-list-mix}
 Let $\cF$ be a class of distributions,  $\alpha,\beta \in (0,1)$, $C> 1$, and $L,m \in \bN$. If $\cF$ is $(C,\alpha)$-accurate $(m,1-\beta,L)$-list-globally-stable learnable, then $\kmix(\cF)$ is $(C,5\alpha)$-accurate $(m_1,1-2k\beta,L_1)$-list-globally-stable learnable, where $L_1 = (\frac{Lk}{\alpha})^k \left(\frac{10ek\log(1/\beta)}{\alpha}\right)^{mk}$, and $m_1=\frac{2mk+8k\log(1/\beta)}{\alpha}$. 
\end{theorem}

Here, we give a high level idea of the proof. First, we show that it is possible to represent the true distribution $g$ as a mixture distribution such that the total mass of its ``far'' components from $\cF$ is small. Now, given a large enough sample from $g$ we can hope that we receive some samples from all non-negligible components of $g$. Assuming we have access to a list globally stable learner algorithm $\cA_1$ for $\cF$, we can apply it on every subset of samples, and approximately recover each ``non-far'' component in $g$. We can then compose all the outputs of the algorithm $\cA_1$ to create a list globally stable learner $\cA_2$ for $\kmix(\cF)$. We show there exist a distribution $\tilde{g}$ in the output of algorithm $\cA_2$ such that it satisfies two properties of Definition~\ref{def:stable-list-decoding}. In other words, we show that
     \textit{(1)} $\cA_2$ is a list globally stable algorithm, and
    \textit{(2)} $\cA_2$ satisfies agnostic utility guarantee.

The following lemma is useful in the construction of our list globally stable learner for mixtures. At a high level, it states that if a distribution $g$ is close to a class of mixtures $\kmix(\cF)$, then $g$ can be expressed as a mixture, such that the overall mass of far components from class $\cF$ is small. 
\begin{lemma}[Lemma 7 of \citet{ashtiani2018sample}]\label{lemma:close-mixture}
    Let $\cF$ be a class of distributions, then any distribution $g$ can be written as a mixture $g=\sum_{i\in[k]}w_i g_i$ such that $w_i\geq 0$, $\sum_{i \in [k]}w_i=1$, and $g_i$'s are distributions satisfying $\sum_{i \in [k]}w_i \dtv(g_i,\cF) = \dtv(g,\kmix(\cF))$.
\end{lemma}

\textbf{Proof of Theorem~\ref{thm:stable-list-mix}:}
\begin{proof} \label{proof:stable-mix}
      Let $g$ be the true distribution.
      Let $g = \sum_{i\in[k]} w_i g_i$ be the representation from Lemma~\ref{lemma:close-mixture} satisfying $\sum_{i \in [k]}w_i \dtv(g_i,\cF) = \dtv(g,\kmix(\cF))$. Let $\cA_1$ be the $(C,\alpha)$-accurate $(m,1-\beta,L)$-list-globally-stable learner for $\cF$.
    Let $S$ be an i.i.d.~sample set from $g$ with size $m_1$. Define $\cL=\{\sum_{j\in [k]}w_j h_j: w\in \hat{\Delta}_k, h_j\in \cH\}$, where $\cH =\bigcup_{\tilde{S} \subseteq S: |\tilde{S}|= m} \cA_1(\tilde{S})$, and $\hat{\Delta}_k$ is an $\frac{\alpha}{k}$-cover for the $(k-1)$-dimensional probability simplex w.r.t.~$\ell_\infty$ from Claim~\ref{lemma:weightcover}. Consider an algorithm $\cA_2$ that receives the i.i.d.~sample set $S$, and outputs $\cL$. We claim that $\cA_2$ is a $(C,5\alpha)$-accurate $(m_1,1-2k\beta,L_1)$-list-globally-stable learner for $\kmix(\cF)$.
    
    Define $I=\{i\in [k]: w_i\geq \frac{\alpha}{k}\}$. For every $i\in  I$, we can write $g = w_i g_i + (1-w_i)\sum_{j\neq i}\frac{w_j}{1-w_i}g_j$. After receiving $N\geq m_1$ samples from $g$, the number of samples coming from $g_i$ has a binomial distribution. Let the corresponding random variable be $X_N$. Since $w_i\geq \frac{\alpha}{k}$, we have $\expect{X_N}/2 \geq m$ and $\expect{X_N} \geq 8 \log(1/\beta)$. Using the Chernoff bound (Theorem 4.5(2) of \citep{mitzenmacher2005probability}), we have $\prob{X_N \leq m} \leq \prob{X_N \leq \expect{X_N}/2} \leq \exp(-\expect{X_N} / 8) \leq \beta$. Meaning that after drawing $N\geq m_1$ samples from $g$, with probability at least $1-\beta$, we will have $m$ samples coming from $g_i$. Using a union bound, with probability at least $1-k\beta$, for all $i \in  I$, there exists $m$ samples coming from $g_i$. Thus, using the fact that $\cA_1$ is a $(m,1-\beta,L)$-list-globally-stable learner for $\cF$, we can with probability at least $1-2k\beta$, for all $i \in  I$, there exists a distribution $\tilde{g}_i\in \cH$ satisfying $\dtv(\tilde{g}_i,g_i)\leq C\cdot\dtv(g_i,\cF)+ \alpha$. For $i \in I$, let $\tilde{w_i}=\frac{1}{\sum_{j \in I} w_j} w_i$. Note that there exists a $w^* \in \hat{\Delta}_k$, such that $|w^*_i - \tilde{w_i}| \leq \alpha/k$ (for $i \in [k]\setminus I$, we set $\tilde{w}_i=w^*_i = 0$). Define $\tilde{g}=\sum_{i\in [k]} w^*_i \tilde{g}_i$ (for $i \in [k]\setminus I$, we set $\tilde{g}_i$ to be an arbitrary element in $\cH$).
    
    Now, we are ready to verify the two properties of list globally stable learning (recall Definition~\ref{def:stable-list-decoding});

    \textbf{Property I}. $\cA_2$ is a list globally stable algorithm.
    
     By construction, $\tilde{g}\in \cL$  with probability at least $1-2k\beta$.
        Also, note that
    \begin{align}
    |\cL|
    & =|\hat{\Delta}_k||\cH|^k
    =(\frac{k}{\alpha})^k \left(L {m_1\choose m}\right)^k
    \leq (\frac{k}{\alpha})^k L^k \left(\frac{2mke+8ke\log(1/\beta)}{\alpha m}\right)^{mk} \\
    & \leq (\frac{Lk}{\alpha})^k \left(\frac{10ek\log(1/\beta)}{\alpha}\right)^{mk}. 
    \end{align}

\textbf{Property II}. $\cA_2$ preserves agnostic utility guarantee:
    \begin{align}
         \dtv(\tilde{g},g)& \leq \dtv(\sum_{i\in [k]} w^*_i\tilde{g}_i,\sum_{i\in [k]} w_ig_i) 
         \leq \frac{1}{2}\sum_{i\in [k]} ||w^*_i \tilde{g}_i-w_ig_i||_1  \\ 
        & \leq \frac{1}{2}\sum_{i\in [k]} ||w^*_i \tilde{g}_i- w_i\tilde{g}_i||_1 + \frac{1}{2}\sum_{i\in [k]}||w_i\tilde{g}_i-w_ig_i||_1\\
        & \leq \sum_{i\in [k]}|w^*_i-w_i|+ \sum_{i\in [k]} w_i\dtv( \tilde{g}_i,g_i)\\
        & \leq  \sum_{i\in [k]}|w^*_i-w_i| +  \sum_{i\in I} w_i\dtv( \tilde{g}_i,g_i)  +\sum_{i\in [k]\setminus I} w_i\dtv(\tilde{g}_i,g_i) \\
        & \leq 
           \sum_{i\in [k]}|w^*_i-w_i| +  \sum_{i\in I} w_i\left(C\cdot\dtv( g_i,\cF)+\alpha\right)  +\sum_{i\in [k]\setminus I} \frac{\alpha}{k}\dtv(\tilde{g}_i,g_i). 
        \end{align}
        Recall that, $\sum_{i \in I}w_i \dtv(g_i,\cF)\leq \sum_{i \in [k]}w_i \dtv(g_i,\cF) = \dtv(g,\kmix(\cF))$. Therefore, we can write:
        \begin{align}
        \dtv(\tilde{g},g)
        & \leq  \sum_{i\in [k]}|w^*_i-w_i| +  C\cdot\dtv(g,\kmix(\cF))  + 2\alpha\\ 
        & \leq  \sum_{i\in [k]}|w^*_i-\tilde{w_i}| + \sum_{i\in [k]}|\tilde{w_i}-w_i|+  C\cdot\dtv(g,\kmix(\cF))  + 2\alpha\\ 
        & \leq  \sum_{i\in [k]}\frac{\alpha}{k} + \sum_{i\in [k]}|\tilde{w_i}-w_i|+  C\cdot\dtv(g,\kmix(\cF))  + 2\alpha\\ 
        & \leq C\cdot\dtv(g,\kmix(\cF))  + 3\alpha  + \sum_{i\in  I}w_i\left(\frac{1}{\sum_{j \in  I} w_j} -1 \right) + \sum_{i\in [k] \setminus I}w_i \\
        &= C\cdot\dtv(g,\kmix(\cF))  + 3\alpha + \left(1-\sum_{i\in I}w_i \right) + \sum_{i\in[k] \setminus I}w_i \\
        &= C\cdot\dtv(g,\kmix(\cF))  + 3\alpha + 2\sum_{i\in [k] \setminus I}\frac{\alpha}{k} \\
        &\leq C\cdot\dtv(g,\kmix(\cF))  + 5\alpha. \qedhere
    \end{align}
    \end{proof}

%% file: GMM.tex
\begin{section}{Agnostic private learning of GMMs}\label{section:GMMS}
As the main application of our framework, we prove the first sample complexity upper bound for privately learning GMMs in the agnostic setting. It is notable that our sample complexity also improves
 the \textit{realizable} result of \citet{afzali2024mixtures} in terms of accuracy parameter by a factor of $\frac{1}{\alpha^2}$.
 
Given our general reduction in Section~\ref{section:main-proof}, it is sufficient to show that GMMs are list globally stable learnable. Indeed, we show that Gaussians are list globally stable learnable and use our the tool from Section~\ref{section:mixtures} to conclude Gaussian mixtures are also list globally stable learnable.

\subsection{List globally stable learning of Gaussians and their mixtures}
We create a list globally stable learner for the class of Gaussians using the robust sample compression schemes of \citet{ashtiani2020near}. At a high level, a class of distributions admits a compression scheme if upon receiving some samples from an unknown distribution, there exists a small subset of those samples that can be used to recover the original distribution up to a reasonable error w.r.t. $\dtv$. 
   At a high level, given a set of samples from an unknown distribution $g$, we run the robust compression algorithm (from Lemma~\ref{lemma:compress-gd}) on every subset of samples. Let $g^* = \argmin_{g' \in \mathcal{G}_d} \mathrm{d}_{\mathrm{TV}}(g,g')$. If $g$ is not too far from the class $\mathcal{G}_d$ (e.g. $\mathrm{d}_{\mathrm{TV}}(g,g^*) \leq \frac{1}{3}$), then $g^*$ can be approximately recovered and is considered a stable candidate. Therefore, it is sufficient to output all discretized neighbors of the recovered elements to include stable elements in the outputted list.
\begin{lemma}[List globally stable learning of Gaussians] \label{lemma:stable-gd}
    Let $\alpha,\beta\in (0,1)$, then $\cG_d$ is $(3,\alpha)$-accurate $(m,1-\beta,L)$-list-globally-stable learnable, where $L=(d\log(1/\beta))^{O(d^2\log(1/\alpha))}$, and $m=O(d\log(1/\beta))$.
\end{lemma}

The formal definition of the robust compression schemes is given below.
\begin{definition}[Definition 4.2 of \citet{ashtiani2020near}, Robust compression]\label{def:compression}
    Let $\tau,t,m:(0,1)\rightarrow \bZ_{\geq 0}$ be functions, $\cF$ be a class of distributions, and $r\geq 0$. We say $\cF$ is $r$-robust $(\tau,t,m)$-compressible, if there exists an algorithm $\cA$ such that for any $\alpha,\beta \in (0,1)$, any $f\in \cF$, and any distribution $g$ satisfying $\dtv(g,f)\leq r$ the following holds:
    \begin{quote}
        Let $S$ be an i.i.d.~sample set from $g$ of size $m(\alpha)\log(1/\beta)$. Then there exists a sequence $L$ of at most $\tau(\alpha)$ samples from $S$, and a sequence $B$ of at most $t(\alpha)$ bits, such that the algorithm $\cA(L,B)$ outputs a distribution satisfying $\dtv(\cA(L,B),f)\leq \alpha$ with probability at least $1-\beta$.
    \end{quote}
\end{definition}

The next lemma asserts that the class of Gaussians is robust compressible. We will later use this compressing algorithm in order to construct a list globally stable learner for Gaussinas.
\begin{lemma}[Lemma 5.3 of \citet{ashtiani2020near}, Robust Compressing Gaussians] \label{lemma:compress-gd}
    Let $\alpha\in(0,1)$, then the class $\cG_d$ is $\frac{1}{3}$-robust $\left(O(d),O(d^2\log(d/\alpha)),O(d)\right)$-compressible. 
\end{lemma}

In order to create a list globally stable learner for a class of distributions, we need the elements outputted by the algorithm to be somehow “discretized”,
the following lemma will later be useful to do so.

\begin{lemma}[Lemma 8.2 of \citet{afzali2024mixtures}]
    \label{cor:gaussians_locally_small_cover}
    For any $0< \alpha \leq \frac{1}{600}$, there exists an $\alpha$-$\dtv$-cover $\cC$ for the class $\cG_d$ satisfying:
    \[
     \sup_{g\in \cG_d}|\{g'\in\cC: \dtv(g',g)\leq 2\alpha\}|\leq 2^{O(d^2)}.
    \]
\end{lemma}

 Putting together, we construct a list globally stable learner for the class of Gaussians. 

\textbf{Proof of Lemma~\ref{lemma:stable-gd}:}
\begin{proof}\label{proof:stable-gd}
    Using Lemma~\ref{lemma:compress-gd}, we know there is a $\frac{1}{3}$-robust $\left(\tau,t,m'\right)$-compressing algorithm $\cA_1$ for $\cG_d$, where $\tau=O(d)$, $t=O(d^2\log(d/\alpha))$, and $m'=O(d)$. Also, let $\cC$ be the $\alpha$-$\dtv$-cover from Lemma~\ref{cor:gaussians_locally_small_cover}, $g$ be the true distribution, and $g_0\in \cC$ be a dummy distribution. Let $S$ be an i.i.d.~sample set of size $m=m'\log(1/\beta)$ from $g$. Consider the set $\cH_1=\{\cA_1(S',B): S'\subseteq S, |S'|\leq\tau, B\in \{0,1\}^t\}$. Next, construct $\cH_2=\{g'\in \cC: \exists h\in \cH_1 \text{ s.t. } \dtv(g',h)\leq 2\alpha  \} \cup \{g_0\}$. We claim that the algorithm $\cA_2$ that takes $S$ as input, and outputs $\cH_2$, is a $(3,\alpha)$-accurate $(m,1-\beta,L)$-list-globally-stable learner for $\cG_d$. To prove this we need to show that for every $g$ there exists a distribution $\tilde{g}$ satisfying the properties of Definition~\ref{def:stable-list-decoding}. Consider the following two cases:

    \textbf{Case 1.} If $\dtv(g,\cG_d)>\frac{1}{3}$. Consider $\tilde{g}=g_0$. Then, it holds that:
    \begin{quote}
        \textit{(1)} $\tilde{g} \in \cH_2$, with probability 1. Since by construction, $g_0$ is always in $\cH_2$.\\
    \textit{(2)} $\dtv(\tilde{g},g)\leq 1 < 3\cdot\dtv(g,\cG_d)$
    \end{quote}

    \textbf{Case 2.} If $\dtv(g,\cG_d)\leq \frac{1}{3}$, let $g^* = \argmin_{g'\in \cG_d}\dtv(g,g')$. Using the definition of robust compression, we know that since $\dtv(g,g^*)\leq \frac{1}{3}$, there exists a distribution $\hat{g}\in \cH_1$, satisfying $\dtv(\hat{g},g^*)\leq \alpha$. Now, let $\tilde{g}=\argmin_{g'\in \cC}\dtv(g^*,g')$. Then it holds that:
     \begin{quote}
        \textit{(1)} $\tilde{g}\in\cH_2$, with probability at least $1-\beta$.\\ Since, $\dtv(\tilde{g},\hat{g})\leq \dtv(\tilde{g},g^*)+\dtv(g^*,\hat{g})=2\alpha$ (Recall that $\cC$ is an $\alpha$-cover).\\
    \textit{(2)} $\dtv(\tilde{g},g)\leq \dtv\left(\tilde{g},g^*\right) + \dtv\left(g^*,g\right) \leq \alpha + \dtv(g,\cG_d)$.
    \end{quote}

    Furthermore, we have $|\cH_2|\leq |\cH_1| \cdot (\sup_{h \in \cH_1} |\{g'\in \cC: \dtv(g',h)\leq 2\alpha\}|) \leq O(m^{\tau+t}) \cdot 2^{O(d^2)}=L$, where the last inequality follows from Lemma~\ref{cor:gaussians_locally_small_cover}, and concludes the desired result.
    \qedhere
\end{proof}

The formal proof is given in Section~\ref{proof:stable-gd}. An immediate corollary of Lemma~\ref{lemma:stable-gd} and Theorem~\ref{thm:stable-list-mix} is that the class of GMMs is list globally stable learnable. 
\begin{corollary} \label{cor:gmm-stable}
        Let $\alpha,\beta\in (0,1)$, then $\kmix(\cG_d)$ is $(3,5\alpha)$-accurate $(m_1,1-2k\beta,L_1)$-list-globally-stable learnable, where $L_1 = (\frac{Lk}{\alpha})^k \left(\frac{10ek\log(1/\beta)}{\alpha}\right)^{mk}$, $m_1=\frac{2mk+8k\log(1/\beta)}{\alpha}$, $L=(d\log(1/\beta))^{O(d^2\log(d/\alpha))}$, and $m=O(d\log(1/\beta))$.
\end{corollary}

\subsection{Agnostic private learning of GMMs}
In the following theorem provide the first sample complexity upper bound for agnostic private learning GMMs.
The proof (stated in \ref{proof-main-gmm}) is the direct result of Theorem~\ref{thm:main-gmm-agnostic} and Corollary~\ref{cor:gmm-stable}. 
\begin{theorem}[Private agnostic learning GMMs]
    \label{thm:main-gmm-agnostic}
     Let $\alpha,\beta,\delta \in (0,1)$ and $\eps \geq 0$. The class $\kmix(\cG_d)$ is $(\eps,\delta)$-privately $21$-agnostic $(n,\alpha,\beta)$-learnable with 
       \[  n=\tilde{O}\left(\frac{k^2d^4+kd^2\log(1/\delta\beta)+\log^2(1/\beta)}{\alpha^2\eps}\right)  .
       \]
\end{theorem}
\begin{proof}\label{proof-main-gmm}
      Corollary~\ref{cor:gmm-stable} implies that $\kmix(\cG_d)$ is $(3,5\alpha)$-accurate $(m_1,0.91,L_1)$-list-globally-stable learnable, where $L_1 = (\frac{Lk}{\alpha})^k \left(\frac{10ek\log(k/0.045)}{\alpha}\right)^{mk}$, $m_1=\frac{2mk+8k\log(k/0.045)}{\alpha}$, $L=(d\log(k/0.045))^{O(d^2\log(d/\alpha))}$, and $m=O(d\log(k/0.045))$ . Putting together with Theorem~\ref{thm:main-private-agnostic}, we get that $\kmix(\cG)$ is $(\eps,\delta)$-DP 21-agnostic $(n,\alpha,\beta)$-learnable with
    \begin{align}
      &
n=\tilde{O}\left(\frac{kd^2+\log(1/\beta\delta)}{\eps} \cdot \frac{kd^2+\log(1/\beta)}{\alpha^2}\right)\\
&
=\tilde{O}\left(\frac{k^2d^4+kd^2\log(1/\delta\beta)+\log^2(1/\beta)}{\alpha^2\eps}\right) 
    \end{align}
    samples.
\end{proof} 

It is worth mentioning that our algorithm is not computationally efficient. Designing an efficient algorithm for learning GMMs, even in the non-private setting, remains an important open question \citep{diakonikolas2017statistical}.

\end{section}

%% file: related.tex
\begin{section}{More on related work}\label{section:related}
In this section, we provide some related work on stability, private distribution learning, and privately learning Gaussian distributions and their mixtures.

Unlike classification, in the distribution learning setting, PAC learnability of general classes of distributions (even in the non-private setting) remains an important open question \citep{diakonikolas2016learning}. Recently, \citet{lechner2023impossibility} showed that there is no single notion of dimension that characterizes the learnability of a given class of distributions. 

Another difference between learning distributions and learning concept classes is discussed in \citet{ben2024distribution}. They show that, unlike classification where realizable and agnostic learning is characterized by the VC dimension of the concept class, there is a class of distributions that is learnable in the realizable setting but not in the agnostic setting. 

Given the connections between robustness and private statistical estimation \citep{dwork2009differential,georgiev2022privacy,liu2022differential,hopkins2023robustness,asi2023robustness,liu2021robust}, it is a natural question to ask if every agnostic learnable class of distributions can be learned privately \citep{afzali2024mixtures}. This conjecture is more likely to hold in the agnostic setting, since there is a partial resolution in \citet{bun2024not} stating that there is a class of distributions that can be learned in the realizable setting with a constant accuracy, but not privately learned with the same level of accuracy.

There is a long line of work trying to demonstrate the private learnability of known classes of distributions, such as Gaussians and their mixtures. \citet{karwa2018finite} presented the first result on the private learnability of unbounded univariate Gaussians. Later, this result was extended to high-dimensional Gaussians with bounded parameters \citep{kamath2019privately,biswas2020coinpress,hopkins2022efficient} and unbounded parameters \citep{aden2021sample}.

In the unbounded setting, although the result of \citet{aden2021sample} was nearly tight, matching the lower bound of \citet{kamath2022new}, it was computationally inefficient. This was later improved in \citet{kamath2022private,kothari2022private,ashtiani2022private}, with the method of \citet{ashtiani2022private} achieving near-optimal sample complexity. The results of \citet{kothari2022private} and \citet{ashtiani2022private} also apply in the robust setting with sub-optimal sample complexity. In the robust setting, the later work of \citet{alabi2023privately} improved the sample complexity in terms of dependence on the dimension. Recently, \citet{hopkins2023robustness} achieved a robust and efficient learner with near-optimal sample complexity for unbounded Gaussians. There are also lower bounds on the sample complexity of private statistical estimations related to Gaussians \citep{portella2024lower, narayanan2023better, kamath2022new, bun2014fingerprinting}.

There has been an extensive line of research on parameter learning and density estimation of Gaussian Mixture Models (GMMs). The goal of parameter learning is to recover the underlying unknown parameters of the GMM, whereas the goal of density estimation is to find a distribution that closely approximates the underlying distribution with respect to $\dtv$. For the parameter learning task (even in the non-private setting), the exponential dependence of the sample complexity on the number of components is inevitable \citep{moitra2010settling}.

There are several works in the private parameter estimation setting for GMMs \citep{nissim2007smooth,vempala2004spectral,chen2023private,kamath2019differentially,achlioptas2005spectral,cohen2021differentially,bie2022private,arbas2023polynomial}.

Unlike parameter estimation, the sample complexity for density estimation can be polynomial in the number of components. In the non-private setting, several results have addressed the sample complexity of learning GMMs \citep{devroye2001combinatorial,ashtiani2018sample}, culminating in the work by \citet{ashtiani2018nearly,ashtiani2020near} that provides the near-optimal bound of $\tilde{\Theta}(kd^2/\alpha^2)$.

In the private setting, one approach would be to create a locally small cover for the class of GMMs and apply the private hypothesis selection method of \citet{bun2019private}. However, this turns out to be impossible, as \citet{aden2021privately} showed that the class of GMMs does not admit a locally small cover. They introduced the first polynomial sample complexity upper bound for learning unbounded axis-aligned GMMs under the constraint of approximate differential privacy (DP). They extended the concept of stable histograms from \citet{karwa2018finite} to learn univariate GMMs. However, this approach cannot be generalized to general GMMs, as it remains unclear how to learn even a single high-dimensional Gaussian using a stability-based histogram.

Recently, \citet{ben2024private} proposed a pure DP method for learning general GMMs, assuming they have access to additional public samples.

Finally, \citet{afzali2024mixtures} proposed the first polynomial sample complexity upper bound of $\tilde{O}(\frac{k^2d^4 \log(1/\delta)}{\alpha^4 \epsilon})$ for privately learning general GMMs in the realizable setting. They show that if one has access to a locally small cover and a list decoding algorithm for a class of distributions (e.g., Gaussians), then mixtures of that class (e.g., GMMs) can be learned privately in the realizable setting. At a high level, a locally small cover is an accurate cover that has a small doubling dimension (is not too dense). A list decoding algorithm is an algorithm that receives some sample from an unknown distribution $g$ and outputs a short list of distributions $\cL$, one of which is very close to $g$. This latter condition is hard to satisfy in the agnostic setting since one cannot hope to recover a heavily corrupted distribution up to a very small error. Moreover, constructing a locally small cover is a delicate matter for the class of high-dimensional distributions (e.g., Gaussians). In contrast, we consider the distinct notion of list global stability and show that it is enough for privately learning a class (even in the agnostic setting). Using this new notion and reduction has two benefits: (1) the underlying class does not need to admit a locally small cover, and (2) there is no need to have an accurate list decoding algorithm (which is not possible when the unknown distribution is heavily corrupted). Moreover, we come up with a list globally stable learner for GMMs, settling the agnostic private learnability of GMMs. Finally, our sample complexity improves their result in terms of dependence on the accuracy parameter by a factor of $\frac{1}{\alpha^2}$.

\subsection{Discussion on stability}\label{sec::stability}
Various notions of stability have been proposed in the context of differential privacy. Global stability \citep{thakurta2013differentially,bun2020equivalence} is one such notion, which was introduced to show the equivalence between online learnability and private learnability of a given concept class. This notion was later refined by \citet{ghazi2021sample,ghazi2021user}. Gloabl stability was further studied in connection with algorithmic replicability \citep{chase2023stability,kalavasis2023statistical}. 
    
\begin{definition}[Global stability \citep{bun2020equivalence}]
     Let $m\in \bN, \rho>0$. We say an algorithm $\cA$ is $(m, \rho)$-globally-stable if for every distribution $\cD$ over input, there exists a hypothesis $h_{\cD}$ such that\\ $\probs{S\sim \cD^m}{\cA(S)=h_{\cD}}\geq \rho$.
\end{definition}

    Global stability requires the algorithm to output \emph{the exact same hypothesis} (with probability $\rho$) when run on i.i.d. data sets. Note that an algorithm that ignores the data set and always outputs the same hypothesis is trivially globally stable. However, we are looking for an algorithm that is both globally stable and accurate (i.e., $h_{\cD}$ should be a ``good'' hypothesis).

    Globally stable algorithms are easy to privatize (e.g., by running them on $O(1/\rho)$ separate i.i.d. sets and using a private histogram 
    \citep{bun2019simultaneous} to aggregate the results).
    For binary classification, it has been shown that this stringent definition of stability is achievable for any online learnable class, establishing the equivalence between online learnability and private learnability \citep{alon2022private}.
    However, this definition is not suitable for for high-dimensional estimation tasks. For example, consider the simple task of mean estimation for a $d$-dimensional Gaussian distribution. Even after discretizing the space of solutions, $\rho$ will be exponentially small in $d$ for any mean estimator that uses poly$(d)$ samples. 
    
    Observing that $\rho$ cannot be generally boosted for globally stable algorithms, \citet{chase2023stability} defined the notion of \emph{list replicability}. Instead of requiring outputting the exact same hypothesis, a list replicabile algorithm can output any member from a fixed (but distribution-dependent) list of outcomes.

    \begin{definition}[List replicability \citep{chase2023stability}]
    Let $m,L\in \bN,\rho>0$. An algorithm $\cA$ is called $(m,\rho,L)$-list-replicable if for every distribution $\cD$ over input, there exists a set of hypotheses $H_{\cD}=\{h_1,h_2,\cdots,h_L\}$ such that
    $\probs{S\sim \cD^m}{\cA(S)\in H_{\cD}}\geq \rho$.
\end{definition}
    
    For the simple task of mean estimation, it is possible design a list replicable algorithm that uses a small number of samples ($m=$poly$(d)$) and with high probability of success ($\rho$ close to 1) albeit with an exponential dependence of $L$ on the dimension. While list replicability has been found useful for studying the relationship between privacy and algorithmic replicability \citep{impagliazzo2022reproducibility, chase2023stability}, it does not seem to be suitable for designing private density estimators:  
    even if we have a list replicable algorithm for learning a class like GMMs, turning it into a private algorithm can blow up the sample complexity.
    To see this, recall that $L$ could be very large (e.g., exponentially large in $d$). Therefore, one would have to run the non-private algorithm on many (distinct) data sets to start seeing repetitive outcomes (i.e., ``collisions''). However, the sample complexity of such an approach will be quite poor, and it is not clear how to privatize a list replicable algorithm otherwise.     
    
    To overcome this challenge, we utilized the related notion of list global stability (Def.~\ref{def:stable-list-decoding-simple}), which was implicitly used in \citet{ghazi2021sample} and later formally defined in \citet{ghazi2021user}. 

        A recent relevant work of \citet{bun2023stability} also uses stability-based techniques for private and replicable agnostic-to-realizable reductions for classification. They use all possible labelings of samples to reduce the agnostic replicable to the realizable setting. However, this is again not applicable in the density estimation setting.
\end{section}

%% file: appendix.tex
\section{Additional facts}

The following simple proposition gives a finite cover for weight vectors used to construct a mixture. 
\begin{claim}\label{lemma:weightcover}
    Let $\alpha \in (0,1]$. There is an $\alpha$-cover for the $(k-1)$-dimensional probability simplex $\{(w_1,w_2,...,w_k)\in \bR_{\geq 0}^k\,:\,\sum_{i\in [k]}w_i=1\}$ w.r.t.~$\ell_\infty$ of size at most $(1/\alpha)^k$.
\end{claim}
\begin{proof}
    Partition the cube $[0,1]^k$ into small cubes of side-length $1/\alpha$. If for a cube $c$, we have $c \cap \Delta_k \neq \emptyset$, put one arbitrary point from $c \cap \Delta_k$ into the cover. The size of the constructed cover is no more than $(1/\alpha)^k$ which is the total number of small cubes.
\end{proof}

\begin{claim}
    \label{claim:failure_prob_inequality}
    Let $x \geq 1$.
    Then $1 + \frac{\log 2}{x} + \frac{\log x}{x} < 2$.
\end{claim}
\begin{proof}
    Let $f(x) = 1 + \frac{\log 2}{x} + \frac{\log x}{x}$.
    Then $f'(x) = -\frac{\log 2}{x^2} + \frac{1 - \log x}{x^2} = \frac{1 - \log(2x)}{x^2}$.
    Note that $f'(x)$ is decreasing so $f$ is concave.
    In addition, $x = e/2$ is the only root of $f'$ so $f$ is maximized at $e/2$.
    Thus, $f(x) \leq f(e/2) = 1 + \frac{2}{e} < 2$.
\end{proof}

\begin{claim}
    \label{claim:choose_inequality}
    Let $c_1 \geq e/2, c_2 > 0$. If $x\geq  4c_1 \log(2c_1) + 2c_2$, then $x \geq c_1 \log(x) + c_2$.
\end{claim}
\begin{proof}
    If $x\geq  4c_1 \log(2c_1) + 2c_2$, then $\frac{x}{2} \geq c_2$. It is sufficient to show $x\geq 2c_1 \log(x)$. Consider the function $f(x) = x - 2c_1 \log(x)$. Then $f'(x) = 1 -\frac{2c_1}{x}$, which implies that for $x> 2c_1$ the function is increasing. As a result, for $x\geq 4c_1 \log(2c_1)$ we have $f(4c_1 \log(2c_1)) = 4c_1 \log(2c_1) -2c_1 \log(4c_1 \log(2c_1))= 4c_1 \log(2c_1)-2c_1\log(2c_1)[1+\frac{\log(2)}{\log(2c_1)}+\frac{\log(\log(2c_1))}{\log(2c_1)}]>0$. The last inequality follows from Claim~\ref{claim:failure_prob_inequality} with $x = \log(2c_1) \geq 1$. Putting together, results in $x\geq c_1 \log(x)+c_2$.
\end{proof}

%% file: main.bbl
\begin{thebibliography}{62}
\providecommand{\natexlab}[1]{#1}
\providecommand{\url}[1]{\texttt{#1}}
\expandafter\ifx\csname urlstyle\endcsname\relax
  \providecommand{\doi}[1]{doi: #1}\else
  \providecommand{\doi}{doi: \begingroup \urlstyle{rm}\Url}\fi

\bibitem[Achlioptas and McSherry(2005)]{achlioptas2005spectral}
Dimitris Achlioptas and Frank McSherry.
\newblock On spectral learning of mixtures of distributions.
\newblock In \emph{International Conference on Computational Learning Theory}, pages 458--469. Springer, 2005.

\bibitem[Aden-Ali et~al.(2021{\natexlab{a}})Aden-Ali, Ashtiani, and Kamath]{aden2021sample}
Ishaq Aden-Ali, Hassan Ashtiani, and Gautam Kamath.
\newblock On the sample complexity of privately learning unbounded high-dimensional gaussians.
\newblock In \emph{Algorithmic Learning Theory}, pages 185--216. PMLR, 2021{\natexlab{a}}.

\bibitem[Aden-Ali et~al.(2021{\natexlab{b}})Aden-Ali, Ashtiani, and Liaw]{aden2021privately}
Ishaq Aden-Ali, Hassan Ashtiani, and Christopher Liaw.
\newblock Privately learning mixtures of axis-aligned gaussians.
\newblock \emph{Advances in Neural Information Processing Systems}, 34:\penalty0 3925--3938, 2021{\natexlab{b}}.

\bibitem[Afzali et~al.(2024)Afzali, Ashtiani, and Liaw]{afzali2024mixtures}
Mohammad Afzali, Hassan Ashtiani, and Christopher Liaw.
\newblock Mixtures of gaussians are privately learnable with a polynomial number of samples.
\newblock In \emph{International Conference on Algorithmic Learning Theory}, pages 47--73. PMLR, 2024.

\bibitem[Alabi et~al.(2023)Alabi, Kothari, Tankala, Venkat, and Zhang]{alabi2023privately}
Daniel Alabi, Pravesh~K Kothari, Pranay Tankala, Prayaag Venkat, and Fred Zhang.
\newblock Privately estimating a gaussian: Efficient, robust, and optimal.
\newblock In \emph{Proceedings of the 55th Annual ACM Symposium on Theory of Computing}, pages 483--496, 2023.

\bibitem[Alon et~al.(2020)Alon, Beimel, Moran, and Stemmer]{alon2020closure}
Noga Alon, Amos Beimel, Shay Moran, and Uri Stemmer.
\newblock Closure properties for private classification and online prediction.
\newblock In \emph{Conference on Learning Theory}, pages 119--152. PMLR, 2020.

\bibitem[Alon et~al.(2022)Alon, Bun, Livni, Malliaris, and Moran]{alon2022private}
Noga Alon, Mark Bun, Roi Livni, Maryanthe Malliaris, and Shay Moran.
\newblock Private and online learnability are equivalent.
\newblock \emph{ACM Journal of the ACM (JACM)}, 69\penalty0 (4):\penalty0 1--34, 2022.

\bibitem[Arbas et~al.(2023)Arbas, Ashtiani, and Liaw]{arbas2023polynomial}
Jamil Arbas, Hassan Ashtiani, and Christopher Liaw.
\newblock Polynomial time and private learning of unbounded gaussian mixture models.
\newblock In \emph{International Conference on Machine Learning}. PMLR, 2023.

\bibitem[Ashtiani and Liaw(2022)]{ashtiani2022private}
Hassan Ashtiani and Christopher Liaw.
\newblock Private and polynomial time algorithms for learning gaussians and beyond.
\newblock In \emph{Conference on Learning Theory}, pages 1075--1076. PMLR, 2022.

\bibitem[Ashtiani et~al.(2018{\natexlab{a}})Ashtiani, Ben-David, Harvey, Liaw, Mehrabian, and Plan]{ashtiani2018nearly}
Hassan Ashtiani, Shai Ben-David, Nicholas Harvey, Christopher Liaw, Abbas Mehrabian, and Yaniv Plan.
\newblock Nearly tight sample complexity bounds for learning mixtures of gaussians via sample compression schemes.
\newblock \emph{Advances in Neural Information Processing Systems}, 31, 2018{\natexlab{a}}.

\bibitem[Ashtiani et~al.(2018{\natexlab{b}})Ashtiani, Ben-David, and Mehrabian]{ashtiani2018sample}
Hassan Ashtiani, Shai Ben-David, and Abbas Mehrabian.
\newblock Sample-efficient learning of mixtures.
\newblock In \emph{Proceedings of the AAAI Conference on Artificial Intelligence}, volume~32, 2018{\natexlab{b}}.

\bibitem[Ashtiani et~al.(2020)Ashtiani, Ben-David, Harvey, Liaw, Mehrabian, and Plan]{ashtiani2020near}
Hassan Ashtiani, Shai Ben-David, Nicholas~JA Harvey, Christopher Liaw, Abbas Mehrabian, and Yaniv Plan.
\newblock Near-optimal sample complexity bounds for robust learning of gaussian mixtures via compression schemes.
\newblock \emph{Journal of the ACM (JACM)}, 67\penalty0 (6):\penalty0 1--42, 2020.

\bibitem[Asi et~al.(2023)Asi, Ullman, and Zakynthinou]{asi2023robustness}
Hilal Asi, Jonathan Ullman, and Lydia Zakynthinou.
\newblock From robustness to privacy and back.
\newblock In \emph{International Conference on Machine Learning}, pages 1121--1146. PMLR, 2023.

\bibitem[Beimel et~al.(2013)Beimel, Nissim, and Stemmer]{beimel2013private}
Amos Beimel, Kobbi Nissim, and Uri Stemmer.
\newblock Private learning and sanitization: Pure vs. approximate differential privacy.
\newblock In \emph{International Workshop on Approximation Algorithms for Combinatorial Optimization}, pages 363--378. Springer, 2013.

\bibitem[Ben-David et~al.(2024{\natexlab{a}})Ben-David, Bie, Canonne, Kamath, and Singhal]{ben2024private}
Shai Ben-David, Alex Bie, Cl{\'e}ment~L Canonne, Gautam Kamath, and Vikrant Singhal.
\newblock Private distribution learning with public data: The view from sample compression.
\newblock \emph{Advances in Neural Information Processing Systems}, 36, 2024{\natexlab{a}}.

\bibitem[Ben-David et~al.(2024{\natexlab{b}})Ben-David, Bie, Kamath, and Lechner]{ben2024distribution}
Shai Ben-David, Alex Bie, Gautam Kamath, and Tosca Lechner.
\newblock Distribution learnability and robustness.
\newblock \emph{Advances in Neural Information Processing Systems}, 36, 2024{\natexlab{b}}.

\bibitem[Bie et~al.(2022)Bie, Kamath, and Singhal]{bie2022private}
Alex Bie, Gautam Kamath, and Vikrant Singhal.
\newblock Private estimation with public data.
\newblock \emph{Advances in Neural Information Processing Systems}, 35:\penalty0 18653--18666, 2022.

\bibitem[Biswas et~al.(2020)Biswas, Dong, Kamath, and Ullman]{biswas2020coinpress}
Sourav Biswas, Yihe Dong, Gautam Kamath, and Jonathan Ullman.
\newblock Coinpress: Practical private mean and covariance estimation.
\newblock \emph{Advances in Neural Information Processing Systems}, 33:\penalty0 14475--14485, 2020.

\bibitem[Bun et~al.(2014)Bun, Ullman, and Vadhan]{bun2014fingerprinting}
Mark Bun, Jonathan Ullman, and Salil Vadhan.
\newblock Fingerprinting codes and the price of approximate differential privacy.
\newblock In \emph{Proceedings of the forty-sixth annual ACM symposium on Theory of computing}, pages 1--10, 2014.

\bibitem[Bun et~al.(2015)Bun, Nissim, Stemmer, and Vadhan]{bun2015differentially}
Mark Bun, Kobbi Nissim, Uri Stemmer, and Salil Vadhan.
\newblock Differentially private release and learning of threshold functions.
\newblock In \emph{2015 IEEE 56th Annual Symposium on Foundations of Computer Science}, pages 634--649. IEEE, 2015.

\bibitem[Bun et~al.(2019{\natexlab{a}})Bun, Kamath, Steinke, and Wu]{bun2019private}
Mark Bun, Gautam Kamath, Thomas Steinke, and Steven~Z Wu.
\newblock Private hypothesis selection.
\newblock \emph{Advances in Neural Information Processing Systems}, 32, 2019{\natexlab{a}}.

\bibitem[Bun et~al.(2019{\natexlab{b}})Bun, Nissim, and Stemmer]{bun2019simultaneous}
Mark Bun, Kobbi Nissim, and Uri Stemmer.
\newblock Simultaneous private learning of multiple concepts.
\newblock \emph{Journal of Machine Learning Research}, 20\penalty0 (94):\penalty0 1--34, 2019{\natexlab{b}}.

\bibitem[Bun et~al.(2020)Bun, Livni, and Moran]{bun2020equivalence}
Mark Bun, Roi Livni, and Shay Moran.
\newblock An equivalence between private classification and online prediction.
\newblock In \emph{2020 IEEE 61st Annual Symposium on Foundations of Computer Science (FOCS)}, pages 389--402. IEEE, 2020.

\bibitem[Bun et~al.(2023)Bun, Gaboardi, Hopkins, Impagliazzo, Lei, Pitassi, Sivakumar, and Sorrell]{bun2023stability}
Mark Bun, Marco Gaboardi, Max Hopkins, Russell Impagliazzo, Rex Lei, Toniann Pitassi, Satchit Sivakumar, and Jessica Sorrell.
\newblock Stability is stable: Connections between replicability, privacy, and adaptive generalization.
\newblock In \emph{Proceedings of the 55th Annual ACM Symposium on Theory of Computing}, pages 520--527, 2023.

\bibitem[Bun et~al.(2024)Bun, Kamath, Mouzakis, and Singhal]{bun2024not}
Mark Bun, Gautam Kamath, Argyris Mouzakis, and Vikrant Singhal.
\newblock Not all learnable distribution classes are privately learnable.
\newblock \emph{arXiv preprint arXiv:2402.00267}, 2024.

\bibitem[Chase et~al.(2023)Chase, Moran, and Yehudayoff]{chase2023stability}
Zachary Chase, Shay Moran, and Amir Yehudayoff.
\newblock Stability and replicability in learning.
\newblock In \emph{2023 IEEE 64th Annual Symposium on Foundations of Computer Science (FOCS)}, pages 2430--2439. IEEE, 2023.

\bibitem[Chen et~al.(2023)Chen, Cohen-Addad, d’Orsi, Epasto, Imola, Steurer, and Tiegel]{chen2023private}
Hongjie Chen, Vincent Cohen-Addad, Tommaso d’Orsi, Alessandro Epasto, Jacob Imola, David Steurer, and Stefan Tiegel.
\newblock Private estimation algorithms for stochastic block models and mixture models.
\newblock \emph{Advances in Neural Information Processing Systems}, 36:\penalty0 68134--68183, 2023.

\bibitem[Cohen et~al.(2021)Cohen, Kaplan, Mansour, Stemmer, and Tsfadia]{cohen2021differentially}
Edith Cohen, Haim Kaplan, Yishay Mansour, Uri Stemmer, and Eliad Tsfadia.
\newblock Differentially-private clustering of easy instances.
\newblock In \emph{International Conference on Machine Learning}, pages 2049--2059. PMLR, 2021.

\bibitem[Devroye and Lugosi(2001)]{devroye2001combinatorial}
Luc Devroye and G{\'a}bor Lugosi.
\newblock \emph{Combinatorial methods in density estimation}.
\newblock Springer Science \& Business Media, 2001.

\bibitem[Diakonikolas(2016)]{diakonikolas2016learning}
Ilias Diakonikolas.
\newblock Learning structured distributions.
\newblock \emph{Handbook of Big Data}, 267:\penalty0 10--1201, 2016.

\bibitem[Diakonikolas et~al.(2017)Diakonikolas, Kane, and Stewart]{diakonikolas2017statistical}
Ilias Diakonikolas, Daniel~M Kane, and Alistair Stewart.
\newblock Statistical query lower bounds for robust estimation of high-dimensional gaussians and gaussian mixtures.
\newblock In \emph{2017 IEEE 58th Annual Symposium on Foundations of Computer Science (FOCS)}, pages 73--84. IEEE, 2017.

\bibitem[Dwork and Lei(2009)]{dwork2009differential}
Cynthia Dwork and Jing Lei.
\newblock Differential privacy and robust statistics.
\newblock In \emph{Proceedings of the forty-first annual ACM symposium on Theory of computing}, pages 371--380, 2009.

\bibitem[Dwork et~al.(2006{\natexlab{a}})Dwork, Kenthapadi, McSherry, Mironov, and Naor]{dwork2006our}
Cynthia Dwork, Krishnaram Kenthapadi, Frank McSherry, Ilya Mironov, and Moni Naor.
\newblock Our data, ourselves: Privacy via distributed noise generation.
\newblock In \emph{Advances in Cryptology-EUROCRYPT 2006: 24th Annual International Conference on the Theory and Applications of Cryptographic Techniques, St. Petersburg, Russia, May 28-June 1, 2006. Proceedings 25}, pages 486--503. Springer, 2006{\natexlab{a}}.

\bibitem[Dwork et~al.(2006{\natexlab{b}})Dwork, McSherry, Nissim, and Smith]{dwork2006calibrating}
Cynthia Dwork, Frank McSherry, Kobbi Nissim, and Adam Smith.
\newblock Calibrating noise to sensitivity in private data analysis.
\newblock In \emph{Theory of Cryptography: Third Theory of Cryptography Conference, TCC 2006, New York, NY, USA, March 4-7, 2006. Proceedings 3}, pages 265--284. Springer, 2006{\natexlab{b}}.

\bibitem[Geng et~al.()Geng, Ding, Guo, and Kumar]{geng2018truncated}
Quan Geng, Wei Ding, Ruiqi Guo, and Sanjiv Kumar.
\newblock Truncated laplacian mechanism for approximate differential privacy.

\bibitem[Georgiev and Hopkins(2022)]{georgiev2022privacy}
Kristian Georgiev and Samuel Hopkins.
\newblock Privacy induces robustness: Information-computation gaps and sparse mean estimation.
\newblock \emph{Advances in Neural Information Processing Systems}, 35:\penalty0 6829--6842, 2022.

\bibitem[Ghazi et~al.(2020)Ghazi, Kumar, and Manurangsi]{ghazi2020differentially}
Badih Ghazi, Ravi Kumar, and Pasin Manurangsi.
\newblock Differentially private clustering: Tight approximation ratios.
\newblock \emph{Advances in Neural Information Processing Systems}, 33:\penalty0 4040--4054, 2020.

\bibitem[Ghazi et~al.(2021{\natexlab{a}})Ghazi, Golowich, Kumar, and Manurangsi]{ghazi2021sample}
Badih Ghazi, Noah Golowich, Ravi Kumar, and Pasin Manurangsi.
\newblock Sample-efficient proper pac learning with approximate differential privacy.
\newblock In \emph{Proceedings of the 53rd Annual ACM SIGACT Symposium on Theory of Computing}, pages 183--196, 2021{\natexlab{a}}.

\bibitem[Ghazi et~al.(2021{\natexlab{b}})Ghazi, Kumar, and Manurangsi]{ghazi2021user}
Badih Ghazi, Ravi Kumar, and Pasin Manurangsi.
\newblock User-level differentially private learning via correlated sampling.
\newblock \emph{Advances in Neural Information Processing Systems}, 34:\penalty0 20172--20184, 2021{\natexlab{b}}.

\bibitem[Hopkins et~al.(2022{\natexlab{a}})Hopkins, Kane, Lovett, and Mahajan]{hopkins2022realizable}
Max Hopkins, Daniel~M Kane, Shachar Lovett, and Gaurav Mahajan.
\newblock Realizable learning is all you need.
\newblock In \emph{Conference on Learning Theory}, pages 3015--3069. PMLR, 2022{\natexlab{a}}.

\bibitem[Hopkins et~al.(2022{\natexlab{b}})Hopkins, Kamath, and Majid]{hopkins2022efficient}
Samuel~B Hopkins, Gautam Kamath, and Mahbod Majid.
\newblock Efficient mean estimation with pure differential privacy via a sum-of-squares exponential mechanism.
\newblock In \emph{Proceedings of the 54th Annual ACM SIGACT Symposium on Theory of Computing}, pages 1406--1417, 2022{\natexlab{b}}.

\bibitem[Hopkins et~al.(2023)Hopkins, Kamath, Majid, and Narayanan]{hopkins2023robustness}
Samuel~B Hopkins, Gautam Kamath, Mahbod Majid, and Shyam Narayanan.
\newblock Robustness implies privacy in statistical estimation.
\newblock In \emph{Proceedings of the 55th Annual ACM Symposium on Theory of Computing}, pages 497--506, 2023.

\bibitem[Impagliazzo et~al.(2022)Impagliazzo, Lei, Pitassi, and Sorrell]{impagliazzo2022reproducibility}
Russell Impagliazzo, Rex Lei, Toniann Pitassi, and Jessica Sorrell.
\newblock Reproducibility in learning.
\newblock In \emph{Proceedings of the 54th Annual ACM SIGACT Symposium on Theory of Computing}, pages 818--831, 2022.

\bibitem[Kalavasis et~al.(2023)Kalavasis, Karbasi, Moran, and Velegkas]{kalavasis2023statistical}
Alkis Kalavasis, Amin Karbasi, Shay Moran, and Grigoris Velegkas.
\newblock Statistical indistinguishability of learning algorithms.
\newblock In \emph{International Conference on Machine Learning}, pages 15586--15622. PMLR, 2023.

\bibitem[Kamath et~al.(2019{\natexlab{a}})Kamath, Li, Singhal, and Ullman]{kamath2019privately}
Gautam Kamath, Jerry Li, Vikrant Singhal, and Jonathan Ullman.
\newblock Privately learning high-dimensional distributions.
\newblock In \emph{Conference on Learning Theory}, pages 1853--1902. PMLR, 2019{\natexlab{a}}.

\bibitem[Kamath et~al.(2019{\natexlab{b}})Kamath, Sheffet, Singhal, and Ullman]{kamath2019differentially}
Gautam Kamath, Or~Sheffet, Vikrant Singhal, and Jonathan Ullman.
\newblock Differentially private algorithms for learning mixtures of separated gaussians.
\newblock \emph{Advances in Neural Information Processing Systems}, 32, 2019{\natexlab{b}}.

\bibitem[Kamath et~al.(2022{\natexlab{a}})Kamath, Mouzakis, and Singhal]{kamath2022new}
Gautam Kamath, Argyris Mouzakis, and Vikrant Singhal.
\newblock New lower bounds for private estimation and a generalized fingerprinting lemma.
\newblock \emph{Advances in Neural Information Processing Systems}, 35:\penalty0 24405--24418, 2022{\natexlab{a}}.

\bibitem[Kamath et~al.(2022{\natexlab{b}})Kamath, Mouzakis, Singhal, Steinke, and Ullman]{kamath2022private}
Gautam Kamath, Argyris Mouzakis, Vikrant Singhal, Thomas Steinke, and Jonathan Ullman.
\newblock A private and computationally-efficient estimator for unbounded gaussians.
\newblock In \emph{Conference on Learning Theory}, pages 544--572. PMLR, 2022{\natexlab{b}}.

\bibitem[Karwa and Vadhan(2018)]{karwa2018finite}
Vishesh Karwa and Salil Vadhan.
\newblock Finite sample differentially private confidence intervals.
\newblock In \emph{9th Innovations in Theoretical Computer Science Conference (ITCS 2018)}. Schloss Dagstuhl-Leibniz-Zentrum fuer Informatik, 2018.

\bibitem[Kothari et~al.(2022)Kothari, Manurangsi, and Velingker]{kothari2022private}
Pravesh Kothari, Pasin Manurangsi, and Ameya Velingker.
\newblock Private robust estimation by stabilizing convex relaxations.
\newblock In \emph{Conference on Learning Theory}, pages 723--777. PMLR, 2022.

\bibitem[Lechner and Ben-David(2023)]{lechner2023impossibility}
Tosca Lechner and Shai Ben-David.
\newblock Impossibility of characterizing distribution learning--a simple solution to a long-standing problem.
\newblock \emph{arXiv preprint arXiv:2304.08712}, 2023.

\bibitem[Liu et~al.(2021)Liu, Kong, Kakade, and Oh]{liu2021robust}
Xiyang Liu, Weihao Kong, Sham Kakade, and Sewoong Oh.
\newblock Robust and differentially private mean estimation.
\newblock \emph{Advances in neural information processing systems}, 34:\penalty0 3887--3901, 2021.

\bibitem[Liu et~al.(2022)Liu, Kong, and Oh]{liu2022differential}
Xiyang Liu, Weihao Kong, and Sewoong Oh.
\newblock Differential privacy and robust statistics in high dimensions.
\newblock In \emph{Conference on Learning Theory}, pages 1167--1246. PMLR, 2022.

\bibitem[McSherry and Talwar(2007)]{mcsherry2007mechanism}
Frank McSherry and Kunal Talwar.
\newblock Mechanism design via differential privacy.
\newblock In \emph{48th Annual IEEE Symposium on Foundations of Computer Science (FOCS'07)}, pages 94--103. IEEE, 2007.

\bibitem[Mitzenmacher and Upfal(2005)]{mitzenmacher2005probability}
M.~Mitzenmacher and E.~Upfal.
\newblock \emph{Probability and Computing: Randomized Algorithms and Probabilistic Analysis}.
\newblock Cambridge University Press, 2005.
\newblock ISBN 9780521835404.
\newblock URL \url{https://books.google.ca/books?id=0bAYl6d7hvkC}.

\bibitem[Moitra and Valiant(2010)]{moitra2010settling}
Ankur Moitra and Gregory Valiant.
\newblock Settling the polynomial learnability of mixtures of gaussians.
\newblock In \emph{2010 IEEE 51st Annual Symposium on Foundations of Computer Science}, pages 93--102. IEEE, 2010.

\bibitem[Narayanan(2023)]{narayanan2023better}
Shyam Narayanan.
\newblock Better and simpler lower bounds for differentially private statistical estimation.
\newblock \emph{arXiv preprint arXiv:2310.06289}, 2023.

\bibitem[Nissim et~al.(2007)Nissim, Raskhodnikova, and Smith]{nissim2007smooth}
Kobbi Nissim, Sofya Raskhodnikova, and Adam Smith.
\newblock Smooth sensitivity and sampling in private data analysis.
\newblock In \emph{Proceedings of the thirty-ninth annual ACM symposium on Theory of computing}, pages 75--84, 2007.

\bibitem[Portella and Harvey(2024)]{portella2024lower}
Victor~S Portella and Nick Harvey.
\newblock Lower bounds for private estimation of gaussian covariance matrices under all reasonable parameter regimes.
\newblock \emph{arXiv preprint arXiv:2404.17714}, 2024.

\bibitem[Thakurta and Smith(2013)]{thakurta2013differentially}
Abhradeep~Guha Thakurta and Adam Smith.
\newblock Differentially private feature selection via stability arguments, and the robustness of the lasso.
\newblock In \emph{Conference on Learning Theory}, pages 819--850. PMLR, 2013.

\bibitem[Vempala and Wang(2004)]{vempala2004spectral}
Santosh Vempala and Grant Wang.
\newblock A spectral algorithm for learning mixture models.
\newblock \emph{Journal of Computer and System Sciences}, 68\penalty0 (4):\penalty0 841--860, 2004.

\bibitem[Yatracos(1985)]{yatracos1985rates}
Yannis~G Yatracos.
\newblock Rates of convergence of minimum distance estimators and kolmogorov's entropy.
\newblock \emph{The Annals of Statistics}, 13\penalty0 (2):\penalty0 768--774, 1985.

\end{thebibliography}
